\documentclass[preprint,12pt]{elsarticle}

\usepackage{geometry}
\usepackage{amssymb}
\usepackage{amsmath}
\usepackage{mathtools}
\usepackage{amsthm}

\usepackage[utf8]{inputenc} 
\usepackage[T1]{fontenc}    
\usepackage{hyperref}       
\usepackage{url}            
\usepackage{booktabs}       
\usepackage{amsfonts}       
\usepackage{nicefrac}       
\usepackage{microtype}      
\usepackage{xcolor}         
\usepackage{bm}
\usepackage{graphicx}
\usepackage{subfigure}

\usepackage{multirow}
\usepackage{subcaption}
\usepackage{adjustbox}
\usepackage[para,online,flushleft]
{threeparttable}
\usepackage{pifont}

\theoremstyle{plain}
\newtheorem{theorem}{Theorem}[section]

\newtheorem{lemma}[theorem]{Lemma}

\theoremstyle{definition}
\newtheorem{definition}[theorem]{Definition}

\theoremstyle{remark}

\usepackage[textsize=tiny]{todonotes}

\journal{Elsevier}

\begin{document}

\begin{frontmatter}



\title{Bidirectional Linear Recurrent Models for Sequence-Level Multisource Fusion}


\author[inst1,inst2]{Qisai Liu}
\author[inst4]{Zhanhong Jiang}
\author[inst3]{Md Zahid Hasan}
\author[inst4]{Joshua R. Waite}
\author[inst5]{Chao Liu}
\author[inst4]{Aditya Balu}
\author[inst1,inst2,inst4]{Soumik Sarkar\corref{cor}}
\ead{soumiks@iastate.edu}

\cortext[cor]{Corresponding author}

\affiliation[inst1]{organization={Department of Mechanical Engineering},
            addressline={Iowa State University}, 
            city={Ames},
            postcode={50011}, 
            state={Iowa},
            country={United States}}

\affiliation[inst2]{organization={Department of Computer Science},
            addressline={Iowa State University}, 
            city={Ames},
            postcode={50011}, 
            state={Iowa},
            country={United States}}

\affiliation[inst3]{organization={Department of Electrical and Computer Engineering},
            addressline={Iowa State University}, 
            city={Ames},
            postcode={50011}, 
            state={Iowa},
            country={United States}}

\affiliation[inst4]{organization={Translational AI Center},
            addressline={Iowa State University}, 
            city={Ames},
            postcode={50011}, 
            state={IA},
            country={United States}}

\affiliation[inst5]{organization={Department of Energy and Power Engineering},
            addressline={Tsinghua University}, 
            city={Haidian},
            postcode={100084}, 
            state={Beijing},
            country={China}}
            
\begin{abstract}
    Sequence modeling is a critical yet challenging task with wide-ranging applications, especially in time series forecasting for domains like weather prediction, temperature monitoring, and energy load forecasting. Transformers, with their attention mechanism, have emerged as state-of-the-art due to their efficient parallel training, but they suffer from quadratic time complexity, limiting their scalability for long sequences. In contrast, recurrent neural networks (RNNs) offer linear time complexity, spurring renewed interest in linear RNNs for more computationally efficient sequence modeling. In this work, we introduce BLUR (Bidirectional Linear Unit for Recurrent network), which uses forward and backward linear recurrent units (LRUs) to capture both past and future dependencies with high computational efficiency. BLUR maintains the linear time complexity of traditional RNNs, while enabling fast parallel training through LRUs. Furthermore, it offers provably stable training and strong approximation capabilities, making it highly effective for modeling long-term dependencies. Extensive experiments on sequential image and time series datasets reveal that BLUR not only surpasses transformers and traditional RNNs in accuracy but also significantly reduces computational costs, making it particularly suitable for real-world forecasting tasks. Our code is available here. 
\end{abstract}




\end{frontmatter}



\section{Introduction}\label{intro}
In modern machine learning, sequence models have played a central role in different domains to precisely model sophisticated underlying relationships from data, such as neural machine translation~\cite{he2018layer,wu2016google}, time series prediction~\cite{zhou2023expanding,hua2019deep,zhou2021informer}, image classification and segmentation~\cite{stollenga2015parallel,tatsunami2022sequencer,liu2017bidirectional,sarkar2014sensor,ahamed2025tscmamba}, and speech recognition~\cite{graves2013hybrid,kheddar2024automatic,mehrish2023review}. Among numerous models developed, two categories of models have received considerable attention and been widely adopted, which are recurrent neural networks (RNNs)~\cite{hermans2013training} and Transformers~\cite{lin2022survey}.

\begin{figure}[t]
    \centering
    \includegraphics[width=0.85\linewidth]{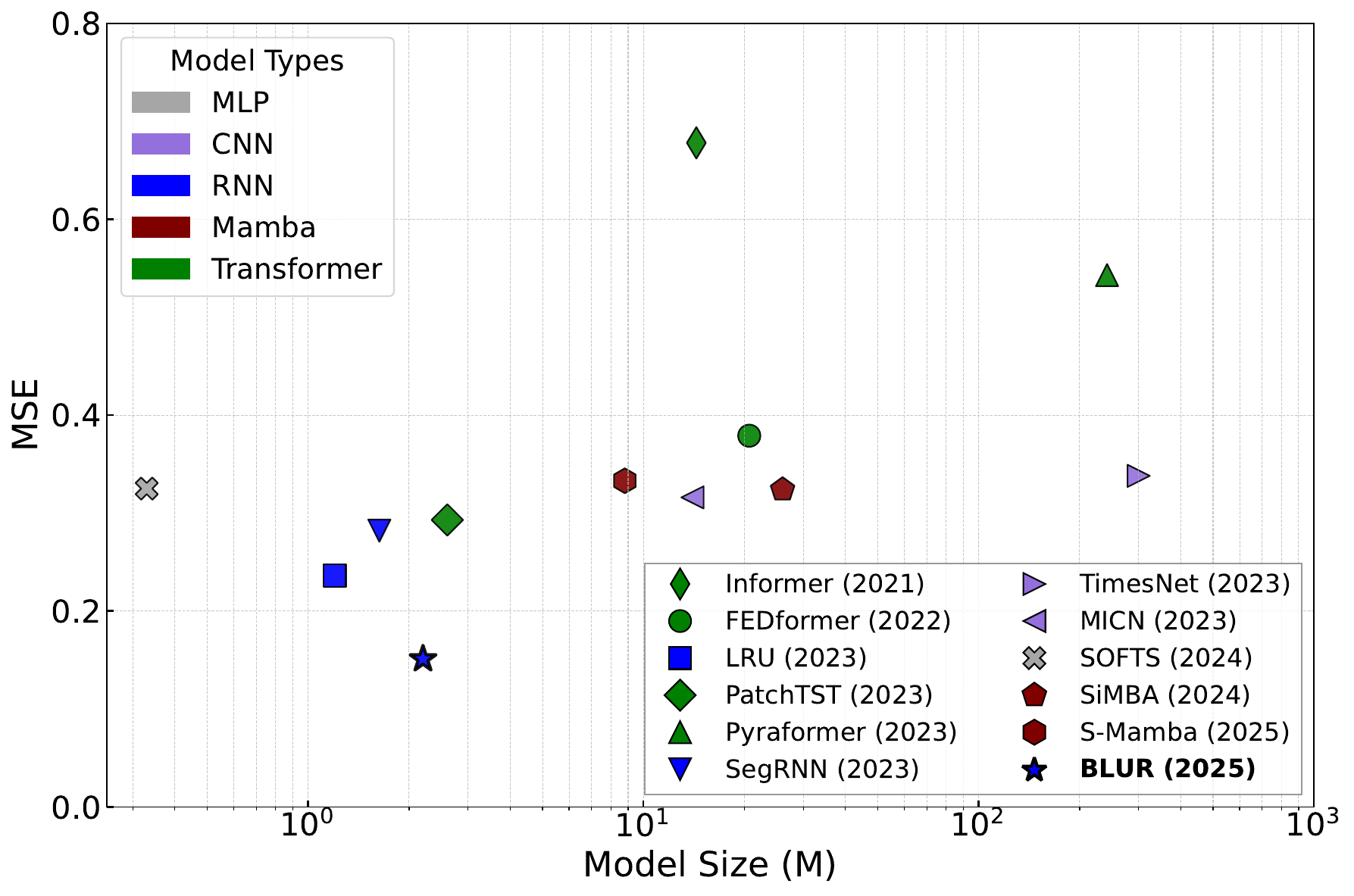}
    \caption{\textit{BLUR} vs. baselines with the number of parameters against the MSE on ETTm$_1$ dataset with horizon=96.}
    \label{fig:param}
    \vspace{5pt} 
\end{figure}

RNN and its variants represent the first generation of sequence models in deep learning, prior to the emergence of transformers. Despite their popularity, the vanishing and exploding gradient problems~\cite{ribeiro2020beyond} have impeded the applications to modeling long-term dependencies. To mitigate this issue, LSTM~\cite{yu2019review} and GRU~\cite{dey2017gate} are developed by leveraging gating mechanisms to control the flow of hidden state information along the sequence. However, they still inevitably undergo slow sequential training speed and limited capabilities in modeling long-term dependencies. In 2017, a seminal work~\cite{vaswani2017attention} devised the attention mechanism leading to the more advanced sequence models, i.e., Transformers. They have swiftly dominated a variety of areas in sequence modeling due to outstanding abilities in parallel training and long-term dependency modeling, but with the cost of quadratic time complexity $\mathcal{O}(N^2)$, where $N$ is assumed to be the sequence length. Unfortunately, this can make the deployment of Transformers on fairly long sequences computationally expensive. In contrast, RNN is able to achieve the linear time complexity $\mathcal{O}(N)$ and is ideally deemed a better model for long-term sequence modeling tasks. Thus, our goal in this work is to establish a model that benefits from parallel training and maintains high computational efficiency.

To this end, researchers have developed a more efficient variant of RNN, termed \textit{linear RNN}, which essentially employs element-wise linear recurrence relations~\cite{martin2017parallelizing}. Motivated by this, the authors in~\cite{gu2021efficiently} introduced the structured state-space models (S4), which achieved stunning performance on complex long-term sequence modeling tasks. Furthermore, S4 and other attempts such as DSS~\cite{gu2022parameterization}, S4D~\cite{gupta2022diagonal}, S5~\cite{smith2022simplified}, and Mamba~\cite{gu2023mamba} have successfully mitigated the $\mathcal{O}(N^2)$ bottleneck caused by the attention layer by resorting to a hidden state with appropriate discretization schemes to model interactions between inputs. The primary advantages of these models lie in two folds. The first one is to remove nonlinearities in recurrence to retain parallel training, and the second is to assume independence between distinct hidden states to realize highly efficient updates. Another line of work on top of linear RNN is to incorporate gating mechanisms~\cite{mehta2022long,peng2023rwkv,wang2022pretraining,tan2022multimodal}, driven by LSTM and GRU, yielding considerable performance gains. Additionally, to study the difference between linear RNNs and deep state-space models, the authors in~\cite{orvieto2023resurrecting} tactically analyzed and ablated a series of changes to standard RNNs, including linearizing and diagonalizing the recurrence. They also introduced a novel architecture called linear recurrent unit (LRU) to fulfill parallel training and fast computation. 

We have also noticed that since the emergence of LRU~\cite{orvieto2023resurrecting} and even its precedents, such as S4, DSS, S4D, S5, and Mamba, few of them have specifically been utilized to perform the relevant timeseries tasks. For example, in~\cite{orvieto2023resurrecting}, all six benchmark datasets are based on images or text, although they possess temporal patterns. However, they are still naturally different from the real timeseries data that can be generated from some unknown highly nonlinear dynamical systems, which already results in the development of numerous Transformer models~\cite{wen2022transformers,zeng2023transformers,wang2024spatiotemporal}. In addition to that, timeseries data could be so noisy that developing a precise model would become extremely challenging, such as in the cases of long-term wind power generation~\cite{ahmadi2020long}, atmospheric circulation~\cite{shepherd2014atmospheric}, and ocean surface current~\cite{sinha2021estimating}.

\textbf{Contributions.} Regardless of the impressive performance of the aforementioned models, they are only able to make use of previous contexts. For some applications, such as speech recognition, where all the utterances are transcribed at once, it is necessary to intelligently exploit the future context as well. By doing so, we can efficiently make full use of past features (via forward pass) and future features (via backward pass) in a specific time frame. Hence, to more effectively capture the long-term dependencies in sequence modeling, we propose a novel model dubbed \textit{Bidirectional Linear Unit for Recurrent (BLUR)} network, which includes forward and backward LRUs. Inheriting from LRU, BLUR still maintains both fast parallel
training and inference thanks to the linear complexity $\mathcal{O}(N)$. BLUR also theoretically enjoys the training stability and universality. Specifically, the contributions of this work are as follows:
\begin{itemize}
    \item Inspired by the bidirectional recurrent models~\cite{schuster1997bidirectional,graves2013speech}, We propose the Bidirectional Linear Unit for Recurrent (BLUR) network by resorting to forward and backward LRUs. This ensures that the information from the past and future states can be captured for each time step to facilitate computationally efficient modeling of long-term dependencies.
    \item We diagonalize the forward and backward recurrence matrices with the eigendecomposition and theoretically analyze the condition to provide assurance for training stability while also applying the universality theory to establish the model approximation property for BLUR. 
    \item To validate the proposed model, several benchmark datasets, including sequential images and timeseries are applied to show efficiency and effectiveness, with the comparison to baselines (as shown in Figure~\ref{fig:param}). It turns out that the bidirectional mechanism is significantly outperforming, with the cost of a slight computational time increase compared to the vanilla LRU. Additionally, we introduce the first benchmark specifically designed for LRU-type models, highlighting their suitability for timeseries prediction and forecasting tasks. See Table~\ref{table:comparison_method} for method comparison.
\end{itemize}

\begin{table}[htp]
\caption{Qualitative comparison among approaches. Full results are in Appendix~\ref{full_comparison_table}.}
\vspace{-0.2in}
\begin{center}
\resizebox{\columnwidth}{!}{ 
\begin{tabular}{c c c}
    \toprule
    \textbf{Method} & \textbf{T} & \textbf{M}\\ \midrule
    RNN~\cite{hewamalage2021recurrent} & $\mathcal{O}(N)$ & Nonlinear recurrence \\
    TCN~\cite{bai2018empirical} & $\mathcal{O}(N)$ & 1D convolutional network \\
    Transformer~\cite{vaswani2017attention}&$\mathcal{O}(N^2)$& Self-attention\\
    Reformer~\cite{kitaev2020reformer} & $\mathcal{O}(N\textnormal{log}N)$ & LSH attention \\
    Informer\cite{zhou2021informer} & $\mathcal{O}(N\textnormal{log}N)$ & Probsparse self-attention\\
    S4~\cite{gu2021efficiently} & $\mathcal{O}(N\textnormal{log}N)$ & Linear recurrence\\
LRU~\cite{orvieto2023resurrecting}   &              $\mathcal{O}(N)$ &Dia-linear recurrence     \\
      SOFTS~\cite{han2024softs} &  $\mathcal{O}(N)$ &  Star aggregate-redistribute  \\
      SiMBA~\cite{patro2024simba} & $\mathcal{O}(N\textnormal{log}N)$ & Mamba + EinFFT\\
      S-Mamba~\cite{wang2025mamba} & $\mathcal{O}(N\textnormal{log}N)$ & Mamba + MLP \\
      \textbf{BLUR}   &     $\mathcal{O}(N)$ & Dia-bi-linear recurrence               \\ 
      \bottomrule
      
\end{tabular}
}
\end{center}
\caption*{\small T: training time complexity, M: mechanism, LSH: locality sensible hashing, Dia-linear: diagonal linear, Dia-bi-linear: diagonal bidirectional linear. EinFFT: Einstein Fast Fourier Transform.}
\label{table:comparison_method}
\end{table}

\vspace{-0.3in}

\section{Related Works}\label{related_works}

In this section, we review the related works to our work. Although a comprehensive review of all relevant literature is beyond the scope of this study, we present a detailed examination of the most relevant research findings.

\textbf{Linear recurrence.}
As mentioned before, one of the weaknesses in vanilla RNNs is the slower sequential training speed, which can primarily be attributed to the full matrix multiplication when updating the hidden states and the impediment of parallel training caused by the nonlinearity within the recurrence. To alleviate the first issue, \cite{lei2017simple} devised an architecture called Simple Recurrent Unit (SRU) to provide the model expressivity and enable highly parallelized implementation with careful initialization to facilitate training. They demonstrated notably the effectiveness and significant efficiency of the model compared to LSTM and convolutional networks. Another work~\cite{martin2017parallelizing} developed a parallel linear recurrence CUDA kernel and applied it to enhance the training and inference speeds by up to 9x.
Analogously, when adopting linear state-space models, diagonalized variants of S4~\cite{gu2022parameterization, gupta2022diagonal} have also expedited the original training speed of S4 by conducting element-wise recurrence. To tackle another challenge in RNNs, recent findings suggest that the nonlinear relationship on past data can be captured by concatenating multiple linear recurrent layers interleaved with nonlinear blocks like multi-layer perceptron (MLP). This essentially eliminates the nonlinearity from the recurrent layer and adds it back in the output layer~\cite{balduzzi2016strongly,gu2021combining}, which is evidently validated in~\cite{smith2022simplified,peng2023rwkv}. Additionally, to theoretically justify the linear recurrence in sequence modeling, two more recent works~\cite{orvieto2023universality,cirone2024theoretical} for the first time revealed that combining MLPs with either real or complex linear recurrences results in arbitrarily precise approximation of regular causal sequence-to-sequence maps. Such a finding implies formally that linear recurrence has universal approximator properties similar to the conclusion for feedforward neural networks~\cite{scarselli1998universal}. 
A more recent work introduced hierarchical gated RNN~\cite{qin2024hierarchically} by developing a gated linear RNN to include forget gates that are lower
bounded by a learnable value. They claimed that this allows the upper layers to model long-term dependencies and the lower layers to model more local, short-term dependencies. One concurrent work~\cite{feng2024were} has recently developed the linear versions of LSTM and GRU, dubbed minLSTM and minGRU, to accomplish efficient parallel training. The advent of these two models advocates the necessity of linear recurrence in revolutionizing traditional nonlinear recurrent models.

\textbf{Bidirectional recurrent models.} Standard RNNs have restrictions in taking input information that is available to the network due to the only forward pass. To break such restrictions and boost the performance, a seminal work~\cite{schuster1997bidirectional} designed Bidirectional RNNs (BRNNs) by connecting two hidden layers of opposite directions to the same output. Since its emergence, BRNNs have extensively been applied to many areas~\cite{arisoy2015bidirectional,chadha2020bidirectional,grisoni2020bidirectional,turek2020approximating} and led to Bidirectional LSTMs (BLSTMs)~\cite{graves2013hybrid,huang2015bidirectional} and Gated Recurrent Units (BGRUs)~\cite{lynn2019deep,liu2022bidirectional}. Though bidirectional architectures have typically yielded performance enhancement, they are still stuck with slow sequential training speed. We have not been aware of any BRNN and its variants to possess fast parallel training as the nonlinearity is still with the recurrence. Xie et al.~\cite{xie2019image} also extended the bidirectional mechanism to attention architecture so as to improve the model performance. Likewise, bidirectional attentions (BAs) were also developed in object classification~\cite{liu2019bidirectional}, image-text matching~\cite{liu2019focus} and multi-object detection~\cite{wang2023banet}. Though BAs are favorably competitive to the state-of-the-art, the quadratic time complexity still remains a bottleneck if $N$ is large, which makes deploying them not practically feasible. Another recently popular bidirectional model is BERT~\cite{devlin2018bert}, which has led to a variety of applications~\cite{sun2019learning,sun2019bert4rec,yue2020bert4nilm,zhang2019hibert}. While BERT involving bidirectional Transformer encoders enables the representation learning more accurately, similar to BAs, the complexity still remains a challenge.

\section{Preliminaries}
\label{preliminaries}

Denote by $\bm{v}:=(\bm{v}_i)_i^N\in\mathcal{V}\subseteq\mathbb{R}^{d\times N}$ the length-$N$ sequences of real $d$-dimensional inputs. Thus, a sequence-to-sequence map between inputs and outputs is defined as a deterministic transformation of input sequences that generates output sequences of the same length denoted by $\bm{y}:=(\bm{y}_i)_i^N\in\mathcal{Y}\subseteq\mathbb{R}^{s\times N}$, where $s$ is the dimension of output. In this context, we only consider the same length for both inputs and outputs for simplicity. However, the developed model architectures allow for varying lengths of inputs and outputs. We also denote by $[N]=\{1,2,...,N\}$. A sequence-to-sequence map is \textit{causal} if for every $i,k\in[N]$, $\bm{y}_k$ is agnostic of inputs $(\bm{v}_i)_{i\geq k+1}$. We resort to $k$ for indicating a time step that is at least $i$ or may be multiple time steps ahead of $i$. The causal sense in this context is only in terms of the input, instead of hidden state~\cite{turek2020approximating}. Therefore, we have the following definition.
\begin{definition}(~\cite{orvieto2023universality})
    A causal sequence-to-sequence map with length-$N$ sequential $d$-dimensional inputs $\bm{v}:=(\bm{v}_i)_{i=1}^N\in\mathcal{V}\subseteq\mathbb{R}^{d\times N}$ and length-$N$ sequential $s$-dimensional outputs $\bm{y}:=(\bm{y}_i)_{i=1}^N\in\mathcal{Y}\subseteq\mathbb{R}^{s\times N}$ is a series of non-linear continuous functions $\mathcal{F}=(\mathcal{F}_k)_{k=1}^N, \mathcal{F}_k:\mathbb{R}^{d\times k}\to\mathbb{R}^s, \forall k\in[N]$ such that $(\bm{v}_i)_{i=1}^N\xmapsto{\mathcal{F}} (\bm{y}_i)_{i=1}^N, \textnormal{s.t.,} \;\bm{y}_k=\mathcal{F}_k((\bm{v}_i)_{i=1}^k)$.
\end{definition}
In this work, we approximate $\mathcal{F}$ by using a neural network. Thereby, we denote by $\tilde{\mathcal{F}}$ the approximation to $\mathcal{F}$, which essentially is a function composition:
\begin{equation}\label{eq_1}
    \tilde{\mathcal{F}}=
    \mathcal{G}\circ\mathcal{L}\circ\mathcal{E}.
\end{equation}
$\mathcal{E}:\mathbb{R}^d\to\mathbb{R}^m$ is a linear embedding layer with biases on inputs. Correspondingly, the encoded sequence is defined as $\bm{u}:=(\bm{u}_i)_{i=1}^N\mathbb{R}^{m\times N}$, where $\bm{u}_k=\mathcal{E}(\bm{v}_k)\in\mathbb{R}^m,\;\forall k\in[N]$. Under some application scenarios, the embedding layer $\mathcal{E}$ is just the identity. $\mathcal{L}: \mathbb{R}^m\to\mathbb{R}^n$ is a linear RNN processing the encoded inputs $(\bm{u}_i)_{i=1}^N$ to produce a sequence of hidden states $(\bm{h}_i)_{i=1}^N\in\mathbb{R}^{n\times N}$, where $n$ is the dimension of $\bm{h}_i$, i.e., $\bm{h}_k=\mathcal{L}(\bm{u}_k)\in\mathbb{R}^n$. Finally, $\mathcal{G}:\mathbb{R}^n\to\mathbb{R}^s$ is typically a non-linear function acting on each hidden state $\bm{h}_k$, parameterized by a model like MLP, i.e., $\hat{\bm{y}}_k=\mathcal{G}(\bm{h}_k)\in\mathbb{R}^s$, $\forall k\in[N]$. Throughout the paper, $\mathcal{G}$ is assumed to be \textit{regular}, which means it is not oscillating too quickly. In general, the goal is to construct a model $\tilde{\mathcal{F}}$ that can transform inputs $(\bm{v}_i)_{i=1}^N$ to the predicted outputs $(\hat{\bm{y}}_i)_{i=1}^N$, which is expected to be close to $\bm{y}=\mathcal{F}(\bm{v})$.

To ensure the universality presented in the latter section, we also assume that each $\mathcal{F}_k$ is a \textit{Barron} function, with a well-defined integrable Fourier transform. Specifically, the Fourier transform of a function $f:\mathbb{R}^p\to\mathbb{R}$ is defined in the following normalization.
\vspace{-0.1in}

\begin{equation}
    \tilde{f}(\omega):=\frac{1}{(2\pi)^p}\int_{\mathbb{R}^p}f(x)e^{-i\langle\omega,x\rangle}dx, 
\end{equation}
which results in the following definition.
\begin{definition}(~\cite{lee2017ability})
    $f$ is a \textit{Barron} function if $\mathcal{C}_f:=\int_{\mathbb{R}^p}\|\omega\|_2|\tilde{f}(\omega)|d\omega<\infty$. $\|\cdot\|_2$ is the $l_2$ norm.
\end{definition}

In the sequel, we introduce the definition of linear RNN to characterize our proposed BLUR.
\begin{definition} (Linear RNN.)
    Suppose that the linear RNN is parameterized by matrices $\mathbf{A}\in\mathbb{R}^{n\times n}$ and $\mathbf{B}\in\mathbb{R}^{n\times m}$. $\mathcal{L}_{\mathbf{A},\mathbf{B}}:\mathbb{R}^{n\times N}\to \mathbb{R}^{m\times N}$ processes an encoded sequence input $\bm{u}\in\mathbb{R}^{m\times N}$ yielding an output sequence of hidden states $\bm{h}\in\mathbb{R}^{n\times N}$ through the following recursive update:
    \begin{equation}\label{eq_2}
        \bm{h}_k = \mathbf{A}\bm{h}_{k-1}+\mathbf{B}\bm{u}_k,
    \end{equation}
    where $\bm{h}_0=\mathbf{0}\in\mathbb{R}^n$.
\end{definition}
Our proposed model is established on top of a recently proposed model LRU such that in the following, we present some necessary background knowledge of LRU. 

\textbf{LRU.} LRU is intrinsically a diagonal linear RNN that is able to address the renowned vanishing gradient issue, which is attributed to the exponentiations of the matrix $\mathbf{A}$ in the following recurrence:

\vspace{-0.2in}

\begin{equation}
    \begin{split}&\bm{h}_1=\mathbf{B}\bm{u}_1,\bm{h}_2=\mathbf{AB}\bm{u}_1+\mathbf{B}\bm{u}_2,\\&\bm{h}_3=\mathbf{A}^2\mathbf{B}\bm{u}_1+\mathbf{AB}\bm{u}_2+\mathbf{B}\bm{u}_3,...
    \end{split}
\end{equation}
\vspace{-0.2in}

In general, the above recurrence can be rewritten as:
\vspace{-0.1in}

\begin{equation}\label{eq_5}
    \bm{h}_k=\sum_{j=0}^{k-1}\mathbf{A}^j\mathbf{B}\bm{u}_{k-j}
\end{equation}
A well-known result~\cite{horn2012matrix} implies that over the space of $n\times n$ non-diagonal real matrices, the
set of non-diagonalizable (in the complex domain) matrices have measured zero.
Since $\mathbf{A}$ is in a space of $n\times n$ non-diagonal real matrices, up to arbitrarily small perturbations, it is diagonalizable over the complex numbers~\cite{orvieto2023resurrecting,orvieto2023universality}. Hence, we have
\begin{equation}
    \mathbf{A}=\mathbf{Q}\mathbf{\Lambda}\mathbf{Q}^{-1}, \;\mathbf{\Lambda}=\textnormal{diag}(\lambda_1,...,\lambda_n)\in\mathbb{C}^{n\times n},
\end{equation}
where $\lambda_1,...,\lambda_n$ are the corresponding eigenvalues of $\mathbf{A}$ and the columns of $\mathbf{Q}\in\mathbb{C}^{n\times n}$ are the corresponding eigenvectors of $\mathbf{A}$. With this in hand, Eq.~\ref{eq_2} is rewritten as 
\begin{equation}
    \bm{h}_k=\mathbf{Q}\mathbf{\Lambda}\mathbf{Q}^{-1}\bm{h}_{k-1}+\mathbf{B}\bm{u}_k
\end{equation}
With some simple mathematical manipulation, we have
\begin{equation}
    \mathbf{Q}^{-1}\bm{h}_k=\mathbf{\Lambda}\mathbf{Q}^{-1}\bm{h}_{k-1}+\mathbf{Q}^{-1}\mathbf{B}\bm{u}_k
\end{equation}
Replacing $\mathbf{Q}^{-1}\bm{h}_k$ and $\mathbf{Q}^{-1}\mathbf{B}$ with $\bm{h}_k$ and $\mathbf{B}$ produces
\begin{equation}\label{eq_9}
    \bm{h}_k=\mathbf{\Lambda}\bm{h}_{k-1}+\mathbf{B}\bm{u}_k,\;\bm{h}_k=\sum_{j=0}^{k-1}\mathbf{\Lambda}^j\mathbf{B}\bm{u}_{k-j}
\end{equation}
\vspace{-0.2in}

which is the key complex-valued diagonal recursion of LRU. One may argue whether the sequence-to-sequence map defined above still holds valid if we adopt such a replacement for the linear RNN. Indeed, it is as $\mathbf{Q}$ is a linear transformation and can be merged with $\mathcal{L}$ and $\mathcal{G}$. Also, the complex eigenvalues play a central role in inducing peculiar memorization properties based on~\cite{orvieto2023universality}.

\begin{figure}
    \centering
    \includegraphics[width=0.7\linewidth]{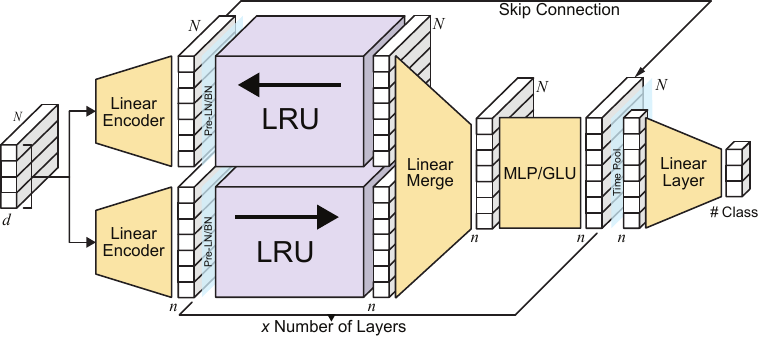}
    \caption{Bidirectional Linear Unit for Recurrent (BLUR) network: the model is a stack of Bidirectional LRU blocks, with two linear encoders, a linear merging layer to merge the bidirectional LRUs, and a nonlinear projection in between, and leverages skip connections and layer (or batch) normalization. The linear encoder is to encode raw inputs to an embedding space. Note also for convenience, we assume that all linear and nonlinear projections in the architecture do not change the hidden dimensions, unlike the dimensions defined in the function composition in Eq.~\ref{blur_function}. Please refer to Figure~\ref{fig:blur_architecture} for more detail.}
    \label{fig:blur}
    \vspace{10pt}
\end{figure}

\section{BLUR}\label{blur}

\subsection{Model architecture}
Motivated by the appealing model expressivity and cheap computational complexity by means of parallel scans~\cite{blelloch1990prefix} of LRU, we propose BLUR, which compromises a forward and a backward LRU, interleaved with a nonlinear layer, as shown in Figure~\ref{fig:blur_architecture}. In the sequel, we start with the bidirectional linear recurrence.
Inspired by BRNN~\cite{schlag2021linear}, the recurrence of hidden states of BLUR can be written as
    \begin{equation}
        \overrightarrow{\bm{h}}_k=\mathbf{\Lambda}_f\overrightarrow{\bm{h}}_{k-1}+\mathbf{B}_f\bm{u}_k,\;
        \overleftarrow{\bm{h}}_k=\mathbf{\Lambda}_b\overleftarrow{\bm{h}}_{k+1}+\mathbf{B}_b\bm{u}_k
    \end{equation}
$\mathbf{\Lambda}_f, \mathbf{B}_f$ are the forward weight matrices and $\mathbf{\Lambda}_b, \mathbf{B}_b$ are the backward weight matrices. Note that the encoded input depends only on the current information $\bm{u}_k$, which is with respect to $(\bm{v}_i)_{i\leq k}$ and blind to any future inputs $(\bm{v}_i)_{i> k}$. This fundamentally satisfies the definition of the causal sequence-to-sequence map. The backward recurrence is also an LRU such that the parallel scans apply immediately. This bidirectional operation can be represented by $\mathcal{L}_1, \mathcal{L}_2:\mathbb{C}^{m}\to\mathbb{C}^n$, respectively. 
One may wonder if this extra layer substantially soars the computational time, compared to the originally single LRU. Empirically, we will show evidently that the increase in computational time is marginal, while the accuracy improvement is significant.

\begin{figure}
    \centering
    \includegraphics[width=0.7\linewidth]{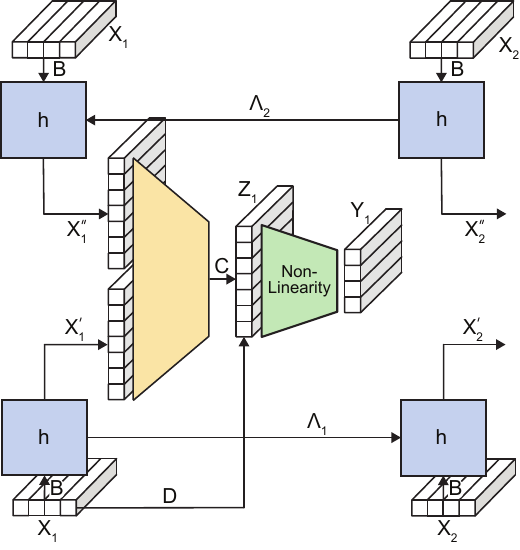}
    \caption{Overview of the BLUR architecture. The model processes input sequences $\{X\}$, which can be sequential images, text, or sensor time series, using forward and backward Linear Recurrent Units (LRUs). Hidden states are updated through diagonal recurrent matrices $\Lambda_1$ and $\Lambda_2$, followed by fusion via a central combination module $C$ and a non-linear transformation. This enables bidirectional context modeling with linear time complexity.}
    \label{fig:blur_architecture}
    \vspace{10pt}
\end{figure}

As discussed in the last section and shown in Figure~\ref{fig:blur}, we need an additional layer to bring the nonlinearity back to the model such that a nonlinear MLP mapping is leveraged accordingly. Different from Eq.~\ref{eq_1}, due to the bidirectional architecture, we first merge both $\overrightarrow{\bm{h}}_k$ and $\overleftarrow{\bm{h}}_k$ with a linear layer and then combines a skip connection from $\bm{u}_k$ such that the following relationship can be acquired:
    \begin{equation}\label{eq_11}
        \hat{\bm{h}}_k=\mathcal{H}(\overrightarrow{\bm{h}}_k,\overleftarrow{\bm{h}}_k)
    \end{equation}
where $\mathcal{H}:\mathbb{C}^n\times\mathbb{C}^n\to\mathbb{C}^n$ is a linear layer to merge hidden states of the opposite directions, $\hat{\bm{h}}_k$ is the merged hidden state, which is key to maintain the consistent dimension as $\overrightarrow{\bm{h}}_k$ or $\overleftarrow{\bm{h}}_k$. $\mathcal{H}$ can also be customized to have nonlinearity.
Moreover, Eq.~\ref{eq_11} is different from $[\overrightarrow{\bm{h}}_k;\overleftarrow{\bm{h}}_k]$ that expands the dimension from $n$ to $2n$~\cite{ma2017dipole}. 
Subsequently, an MLP layer (or GLU) is adopted in this context such that
\begin{equation}
    \bm{z}_k=\mathcal{G}(\hat{\bm{h}}_k),
\end{equation}
where $\mathcal{G}:\mathbb{R}^n\to\mathbb{R}^s$, where we take the real part of $\hat{\bm{h}}_k$ for the computation, as suggested in~\cite{orvieto2023resurrecting}. The output nonlinearity can be made as complex as one wants. The reason why we use MLP is because its universality has been theoretically validated, which will accommodate our analysis. We then resort to skip connections with $\bm{u}_k$ to empirically boost the performance such that,
\begin{equation}
    \bm{y}_k=\mathcal{L}_3(\bm{z}_k+\mathbf{C}\bm{u}_k),
\end{equation}
where $\mathcal{L}_3:\mathbb{R}^n\to\mathbb{R}^s$ is the final linear layer to map hidden states to outputs, $\mathbf{C}\in\mathbb{R}^{n\times m}$.
As Eq.~\ref{eq_1} only supports a single linear RNN layer and has no merging layer, it does not suffice to represent the BLUR model. Consequently, in the next, the new function composition is tailored for BLUR by adapting Eq.~\ref{eq_1}.
\begin{equation}\label{blur_function}
    \hat{\mathcal{F}}=\mathcal{L}_3\circ\mathcal{G}\circ\mathcal{H}\circ(\mathcal{L}_1+\mathcal{L}_2)\circ\mathcal{E}.
\end{equation}
In this context, the embedding layer $\mathcal{E}$ encodes the raw inputs into a higher-dimensional embedding space, which evidently helps enhance performance. Such an embedding may particularly have a significant impact on encoded inputs when raw inputs possess intricate features. Note that the function composition in Eq.~\ref{blur_function} does not show all layers in Figure~\ref{fig:blur}, including the layer (or batch) normalization and the time pooling layers. The architecture shown in Figure~\ref{fig:blur} is what we implement empirically, while the function composition focuses primarily on the key components in BLUR, though we can certainly expand it to include all layers.

\subsection{Theoretical analysis}
In this subsection, we theoretically analyze the stability and universality of the proposed BLUR. 
One may wonder if the analysis for $\tilde{\mathcal{F}}$ is proper to the new function composition of BLUR, $\hat{\mathcal{F}}$. The answer is affirmative, as LRU is linear, which allows us even to design $d$ independent LRUs acting on separate input dimensions. If we assume in this context that $d$ is an even number such that $d/2$ LRUs are forward and the rest of LRUs are backward, then we achieve the bidirectional mechanism. Without loss of generality, we can design opposite directions of LRUs for the same input. Such LRUs can even be combined into one single block using a properly designed $\mathbf{B}$ that takes $\mathbf{B}_f$ and $\mathbf{B}_b$.

\textbf{Stability.}
Recall the recurrence of LRU in Eq.~\ref{eq_9} and adapt it as follows

\vspace{-0.2in}

\begin{equation}\label{eq_15}
    (\textnormal{Forward})\; \bm{h}_k=\sum_{j=0}^{k-1}\mathbf{\Lambda}^j_f\mathbf{B}_f\bm{u}_{k-j}
\end{equation}

\vspace{-0.1in}

The norm of component $j$ of $\bm{h}$ at the time step $k$ evolves such that $\mathcal{O}(|\bm{h}_{k,j}|)=\mathcal{O}(|\lambda_{f,j}|^k)$, where $\lambda_f$ signifies the eigenvalues for the forward recurrence. Similarly, the backward recurrence extending from $k$ to 1 (by setting $\bm{h}_{k+1}=\mathbf{0}$) can be unrolled and represented by
\begin{equation}
    \begin{split}
    &\bm{h}_{k}=\mathbf{B}\bm{u}_{k},\bm{h}_{k-1}=\mathbf{\Lambda B}_b\bm{u}_{k}+\mathbf{B}\bm{u}_{k-1},\\&\bm{h}_{k-2}=\mathbf{\Lambda}_b^2\mathbf{B}\bm{u}_{k}+\mathbf{\Lambda B}_b\bm{u}_{k-1}+\mathbf{B}\bm{u}_{k-2},...
    \end{split}
\end{equation}
such that
\vspace{-0.3in}

\begin{equation}\label{eq_17}
        (\textnormal{Backward})\; \bm{h}_1=\sum_{j=0}^{k-1}\mathbf{\Lambda}^j_b\mathbf{B}_b\bm{u}_{j+1},
\end{equation}
which enables the evolution of the norm of component $j$ of $\bm{h}$ at the corresponding time step 1 to be $\mathcal{O}(|\bm{h}_{1,j}|)=\mathcal{O}(|\lambda_{b,j}|^{k})$, where $\lambda_{b}$ indicates the eigenvalues for the backward sequence. Thereby, for both directions, a sufficient condition to ensure the model stability is $|\lambda_{f,j}|, |\lambda_{b,j}|<1$ for all $j$~\cite{gu2021efficiently}. One may argue that if the forward LRU is stable, then the backward LRU should also be stable. This would degenerate the sufficient condition of the stability to $|\lambda_{f,j}|<1$. However, we empirically initialize both models separately, so the forward stability may not necessarily provide a guarantee for the backward stability.

Through the above analysis, we have realized that the eigenvalues of the linear recurrent weight matrices $\mathbf{\Lambda}_f$ and $\mathbf{\Lambda}_b$ are critical for model learning. When initializing the BLUR, we need to assume that the eigenvalues of both of recurrent weight matrices are distinct. Luckily, we can also control the eigenvalues with a proper initialization. Since our goal is to maintain the absolute values of all eigenvalues $\lambda_f$s and $\lambda_b$s within 1, one can simply sample them uniformly from a complex ring in between $[e_{min}, e_{max}]\subseteq [0,1]$, where $e_{min}$ is the minimum absolute value of eigenvalue and $e_{max}$ the maximum. Specifically, we have the following relationship:
$
    \lambda_{f,i},\lambda_{b,i}\sim \mathbb{T}[e_{min}, e_{max}]:=\{\lambda\in\mathbb{C}|e_{min}\leq|\lambda|\leq e_{max}\}.
$
Motivated by this, the initialization of $\mathbf{\Lambda}_f$ and $\mathbf{\Lambda}_b$ can be close to the unit circle as much as possible.
Such a simple trick can help significantly mitigate the issue of $\bm{h}_k$ exploding while aiding in the long-term dependency modeling. In~\cite{orvieto2023resurrecting}, to reinforce the model training stability, the authors proposed the exponential parameterization for the complex eigenvalues, which also significantly enhances model performance in some experiments. The vanishing gradient issue also becomes one incentive to apply this since when taking powers of eigenvalues close to the origin, the signal from past inputs subsides gradually. In this work, we follow the similar stabilization as done in~\cite{orvieto2023resurrecting} for consistency on the adopted LRU. Apart from stability, the model universality of BLUR is another topic we need to investigate to provide a guarantee for the sufficient approximation of BLUR to the underlyingly complex data patterns. In what follows, by leveraging the classical universality theory from~\cite{barron1993universal} and the recently developed theory from~\cite{orvieto2023universality}, we analyze the universality of BLUR.

\textbf{Universality.}
We will show that if the model $\hat{\mathcal{F}}=\mathcal{L}_3\circ\mathcal{G}\circ\mathcal{H}\circ(\mathcal{L}_1+\mathcal{L}_2)\circ\mathcal{E}$ grows in width, we can always parameterize $\mathcal{F}$ using $\hat{\mathcal{F}}$ such that the approximation error is within a pre-defined desired accuracy $\varepsilon>0$.
\begin{theorem}\label{theorem_1}
    Suppose that the inputs $\bm{v}=(\bm{v}_i)_{i=1}^N\in\mathcal{V}\subseteq\mathbb{R}^{d\times N}$ are bounded, i.e., $\|\bm{v}_i\|<\infty$ for all $i\in[N]$. Denote by $\textnormal{dim}(\mathcal{V})$ the vector-space dimension of $\mathcal{V}$. Let $\mathcal{L}_1$ and $\mathcal{L}_2$ be LRU with either real or complex eigenvalues and with width $n\geq \textnormal{dim}(\mathcal{V})$. Let the nonlinear mapping $\mathcal{E}$ be parameterized by an MLP with width $D\geq \mathcal{O}(nN/\varepsilon^2)$. Hence, $\hat{\mathcal{F}}=\mathcal{L}_3\circ\mathcal{G}\circ\mathcal{H}\circ(\mathcal{L}_1+\mathcal{L}_2)\circ\mathcal{E}$ approximates pointwise $\mathcal{F}$ with error $\varepsilon$:
$ \textnormal{sup}_{\bm{v}\in\mathcal{V}}\|\hat{\mathcal{F}}(\bm{v})-\mathcal{F}(\bm{v})\|\leq \varepsilon$.
\end{theorem}
$\textnormal{dim}(\mathcal{V})$ in Theorem~\ref{theorem_1} can be simply set as $dN$. Due to the limit of the space, we provide the proof sketch here while deferring the proof to Appendix~\ref{additional_proof}. The proof is fairly non-trivial and technical. 

\textit{Proof sketch:}
The first step is to show that both forward and backward LRUs are lossless compressors, which means that the LRU outputs $\overrightarrow{\bm{h}}_k$ and $\overleftarrow{\bm{h}}_k$ are able to reconstruct the input $(\bm{v}_i)_{i=1}^k$, assuming $n$ is sufficiently large. If $\overrightarrow{\bm{h}}_k$ and $\overleftarrow{\bm{h}}_k$ can reconstruct the inputs, it is immediately obtained that $\hat{\bm{h}}_k$ is also able to do so. Another step is to ensure the nonlinear mapping MLP (even some other functions such as GLU~\cite{dauphin2017language}) on top of $\hat{\bm{h}}_k$ can reconstruct the group truth mapping $\mathcal{F}_k$, namely, $\mathcal{F}_k((\bm{v}_i)_{i=1}^k)\simeq \mathcal{L}_3\circ\mathcal{G}(\hat{\bm{h}}_k)$ (we ignore the skip connections in this context for simplicity), by assuming that the number of hidden MLP neurons $D$ satisfies the lower bound $D\geq \mathcal{O}(nN/\varepsilon^2)$. The second step of the proof also requires critical results from~\cite{barron1993universal} to show the universality of one hidden-layer MLP. Please look at Appendix~\ref{additional_proof} for the comprehensive steps.

Theorem~\ref{theorem_1} advocates an adequate approximation of BLUR to the real nonlinear function that maps an input sequence to an output sequence. It also implies the explicit relationship between the width of the nonlinear mapping $D$ and the hidden dimension $n$. This study focuses primarily on MLP due to its analysis maturity, while other nonlinear activation functions can also be used with adjusted constants.
Additionally, we do not make the assumption that the target mapping $\mathcal{F}$ is itself generated from an underlying dynamical system. This is different from previous works~\cite{li2005approximation,schafer2006recurrent,hardt2018gradient}, where an assumption has been imposed that $\mathcal{F}$ is induced from some unknown dynamical systems. Their obtained results usually resemble those in feed-forward neural networks. On the other hand, in our work, we consider general input-output relationships of temporal sequences and do not require the collected data to be sampled from special processes, such as the partially observed Markov process. We also notice that a recent work~\cite{li2022approximation} has performed a systematic study on the approximation property of linear RNNs but in the continuous time domain. They arrive at the similar approximation property for linear RNNs by leveraging functional theory and the classical Riesz-Markov-Kakutani representation theorem~\cite{yaremenko2023generalization}, different from the Universal Approximation Theorem by Barron~\cite{barron1993universal}. In the proof, they also require the mapping $\mathcal{F}$ to be continuous, linear, causal, regular, and time-homogeneous, which seemingly coincides with ours. Additionally, their recurrent weight matrix is not exactly diagonal, though its eigenvalues are assumed to have a negative real part to ensure stability. In this scenario, the updating of the hidden state still involves full matrix multiplication, resulting in a slower training speed.


\begin{figure}[htb!]
\centering
\begin{subfigure}
    \centering
    \includegraphics[width=\linewidth]{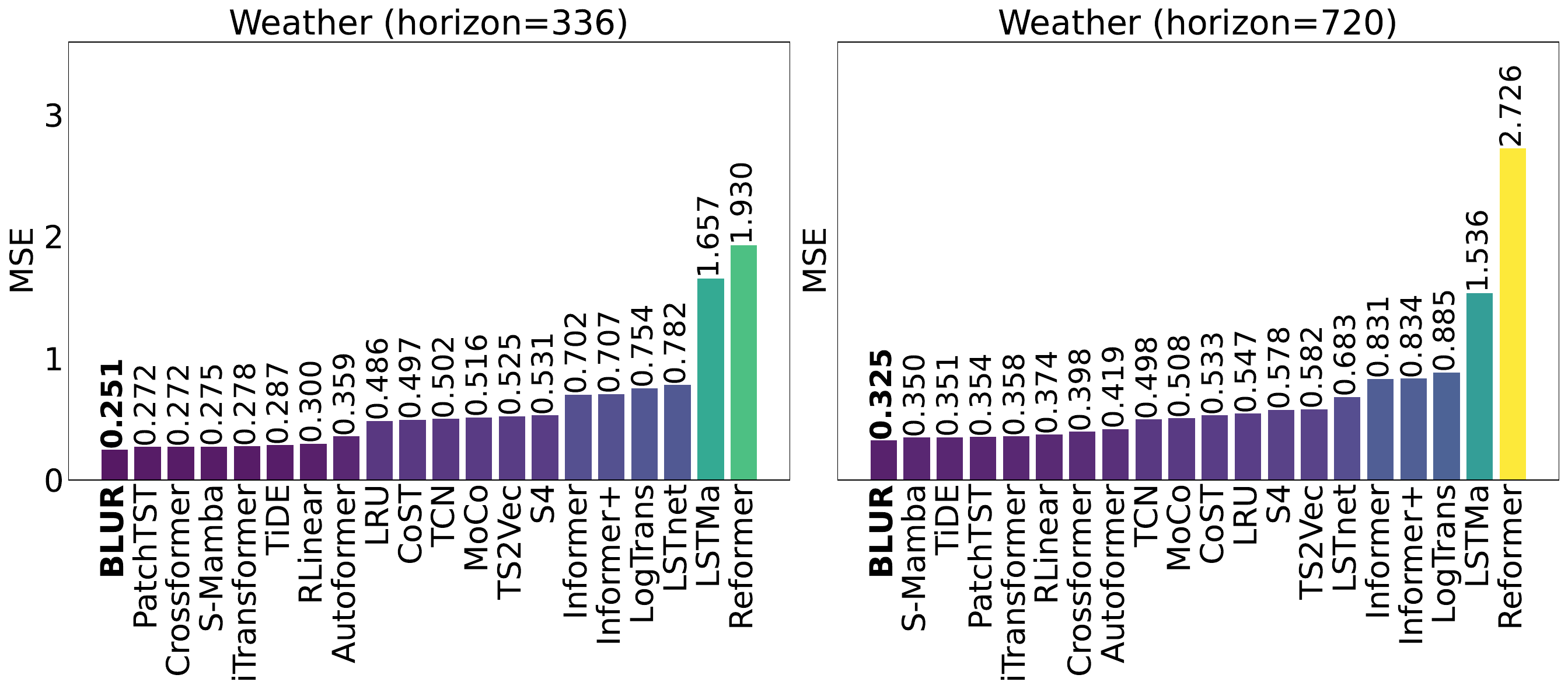}
    \captionsetup{skip=5pt} 
    \caption{MSE comparison with horizons 336 and 720 for Weather.}
    \label{fig:weather_comparison}
\end{subfigure}
\hfill
\vspace{0.1in}
\begin{subfigure}
    \centering
    \includegraphics[width=\linewidth]{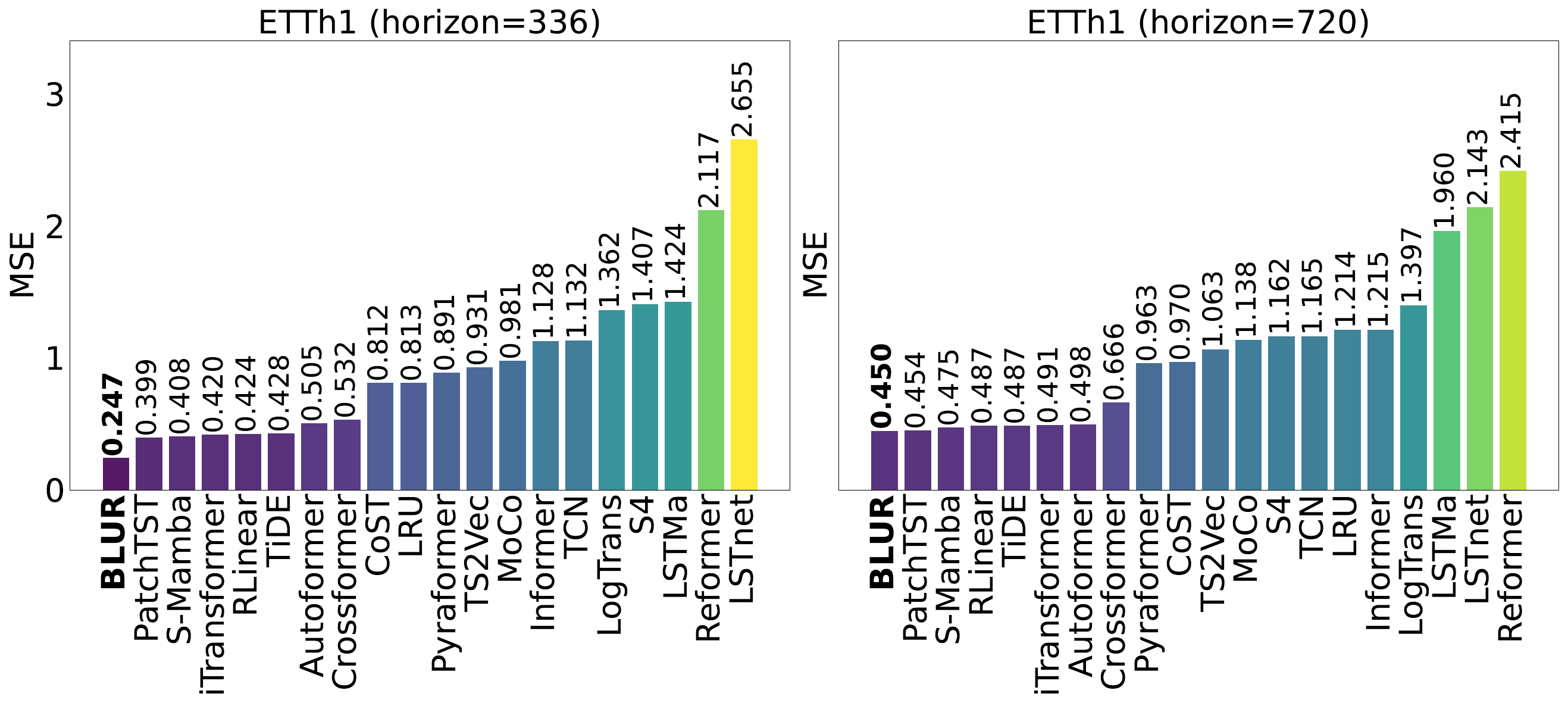}
    \captionsetup{skip=5pt} 
    \caption{MSE comparison with horizons 336 and 720 for ETTh$_1$.}
    \label{fig:etth1_comparison}
\end{subfigure}
\vspace{5pt}
\label{fig:weather_etth1_comparison}
\end{figure}

\section{Numerical Results}
\label{results}

In this section, we present empirical results for the validation of BLUR resorting to two categories of benchmark datasets, comprising sequential images and text and sensory timeseries. The first category includes ListOps~\cite{tay2020long}, Text~\cite{orvieto2023resurrecting}, sequential CIFAR, sequential MNIST, permuted sequential MNIST~\cite{smith2022simplified}. The first three are from the \textit{Long Range Arena} benchmark suite~\cite{tay2020long}. For timeseries, we directly utilize the same benchmark datasets listed in~\cite{zhou2021informer}, which have been popularly used for testing new sequencing models. In this study, we aim to evaluate BLUR on the datasets by comparing it to the recently proposed LRU as well as its precedents, such as S4 and S5. Our primary focus of the evaluation is to inspect if BLUR can benefit from the bidirectional mechanism. This is partially motivated by pairs of RNN vs. BRNN and LSTM vs. BLSTM 
Therefore, for the sequential image and text prediction task, we will have the LRU and S5 as baselines.
While for the timeseries prediction task, we will have an extensive comparison between LRU/BLUR with many baselines, including S4 and Informer. Note also that LRU has never been assessed on timeseries datasets. The evaluation for the suitability of LRU-type models on these datasets can be of independent interest and regarded as an empirical contribution. We refer interested readers to Appendix~\ref{models} and~\ref{datasets} for more details about models and datasets. To achieve fast parallel training, an algorithm called \textit{parallel scan}~\cite{heinsen2023parallelization} has been widely employed in many linear recurrent models. We also exploit it in this study and include some introduction in Appendix~\ref{parallel_scan}. All results are the average based on five random runs. Note also that for the timeseries prediction task, recent works~\cite{han2024softs,patro2024simba} primarily reported results with medium to long horizons (from 96 to 720), while this work complies with the earlier work~\cite{zhou2021informer} to have a more comprehensive spectrum in the horizon (from 24 to 720). We summarize references for all baselines in Appendix~\ref{references}.

\begin{figure}[h!]
    \centering
    \includegraphics[width=1\linewidth]{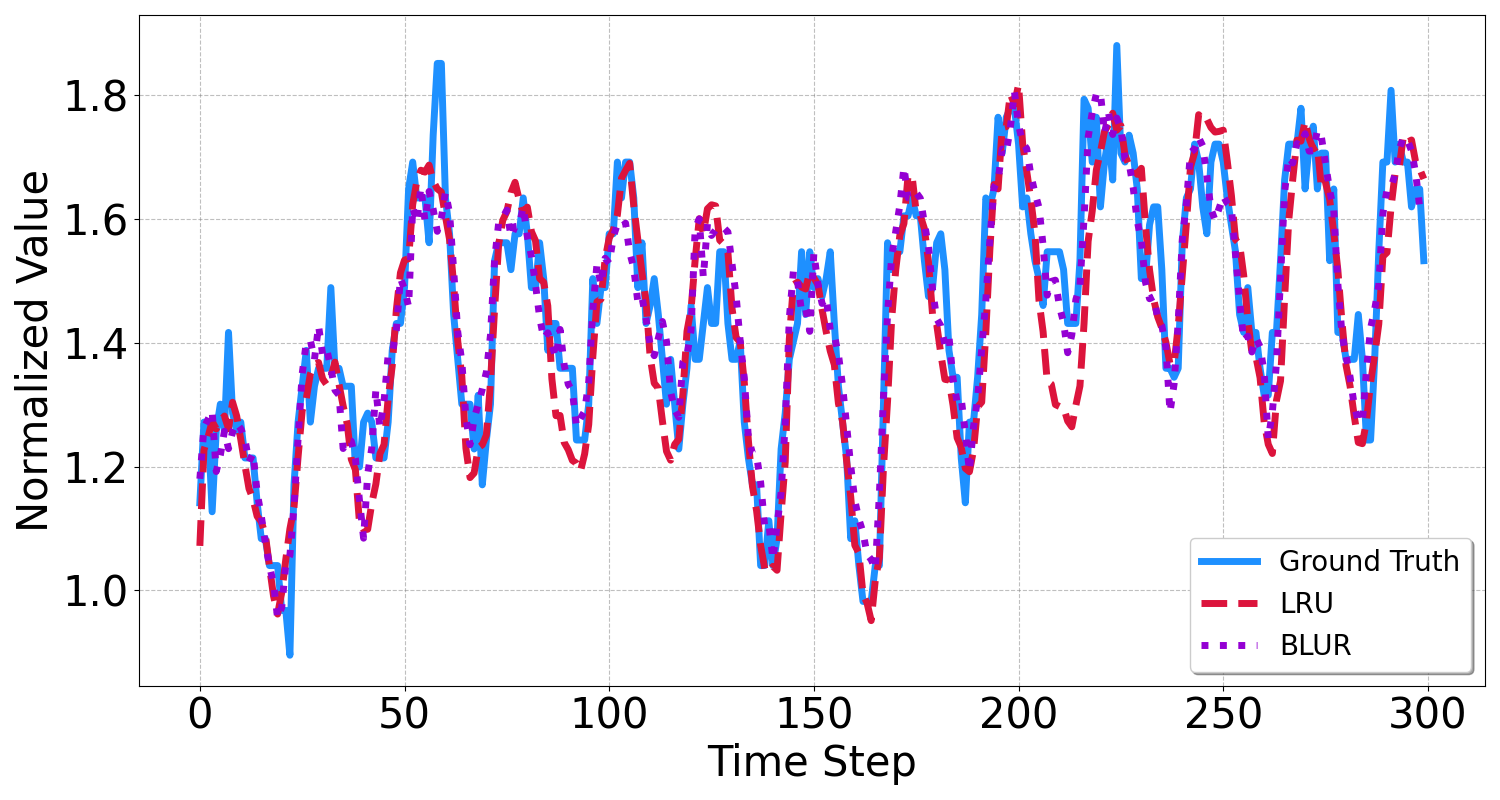}
    \caption{A snippet of comparison between ground truth and predictions of target variable by LRU/BLUR on Weather data with a horizon of 24.}
    \label{fig:ground_truth_prediction}
\end{figure}

\textbf{Sequential images and text.} 
In Table~\ref{table:seq_images}, the results are shown among LRU, S4, S5, and BLUR evaluated on multiple sequential images and text datasets. Values are shown with the metric of classification accuracy, i.e., predicting correct labels. For the parameter setting in model architectures, BLUR uses similar ones adopted in LRU. Overall, BLUR consistently outperforms LRU and S5, except for one task where it falls slightly behind S5. It surpasses S4 in most tasks, demonstrating that the bidirectional mechanism effectively captures the full context of the input sequence, leading to a significant improvement in prediction accuracy.
This necessitates the bidirectional mechanism in challenging sequence modeling tasks. We include the full table in Appendix~\ref{additional_results}.


\begin{table}[htp]
\caption{Performance comparison (\%) between LRU and BLUR on sequential image and text datasets.}
\vspace{-0.1in}
\begin{center}
\begin{tabular}{c c c c c c}
    \toprule
    \textbf{Method} & \textbf{sC} & \textbf{LO} & \textbf{TE} & \textbf{sM} & \textbf{psM}\\ \midrule
      LRU   &        89.0            & 54.8    & 86.7& 99.3& \textbf{98.1}\\
      S4*~\cite{gu2021efficiently}& \textbf{91.1}& \textbf{59.6}& 86.8& 90.9&94.2\\
      S5*~\cite{orvieto2023resurrecting} & 88.8 & 58.5& 86.2& 88.9&95.7\\
      \textbf{BLUR}   & 90.0  & 54.6                     &\textbf{87.3}    & \textbf{99.6}  & \textbf{98.1}       \\ 
      \bottomrule
      
\end{tabular}
\end{center}

\label{table:seq_images}
\caption*{\small sC: sCIFAR, LO: ListOps, TE: Text, sM: sMNIST, psM: psMNIST, * means the paper results.}
\end{table}

\textbf{Timeseries.} 
To thoroughly assess the capabilities of the proposed BLUR, we leverage four datasets with different prediction horizons ranging from short to long terms. We also select S4 and Informer as two baselines. Since the former is the outstanding representative of the linear RNN models and the latter is the delegate of the timeseries transformer model. Their results are included in the main text. To more comprehensively assess the performance of BLUR, we also compare it to many recently developed models over a few years in Appendix~\ref{additional_results}.
Different datasets have slightly different horizons. For ETTh and Weather datasets, the prediction horizons are \{1d, 2d, 7d, 14d, 30d, 40d\}, while ETTm \{6h, 12h, 24h, 72h, 168h\}.
Note that the input past sequence length for each scenario is the same as the prediction horizon. We also resort to two evaluation metrics, including MAE $=\frac{1}{N}\sum_{i=1}^N|\bm{y}-\hat{\bm{y}}_i|$ and MSE $=\frac{1}{N}\sum_{i=1}^N(\bm{y}-\hat{\bm{y}}_i)^2$.


Overall, as shown in Table~\ref{table:timeseries}, the proposed BLUR model significantly outperforms LRU, S4, and Informer. Particularly, compared to S4 and Informer, the model performance improvement from BLUR with the short and medium horizons is extremely remarkable while remaining superior or competitive in long horizons. Additionally, LRU itself also surpasses S4 and Informer in most scenarios. On the one hand, the results in Table~\ref{table:timeseries} suggest that \textit{LRU type of models excel in timeseries prediction tasks and are compelling in modeling long-term dependencies in sequencing tasks}. On the other hand, the comparison between BLUR and LRU strengthens again that predictions can benefit strikingly from the bidirectional mechanism that allows to take both past and future state information at a time step, thus leading to a more effective sequencing model. To compare BLUR with recent models, including other transformer variants and Mamba models, we present additional results in Figure~\ref{fig:weather_comparison} and Figure~\ref{fig:etth1_comparison}, which suggests the superiority of BLUR over other baselines with long horizons.
Figure~\ref{fig:ground_truth_prediction} depicts the ground truth and predicted outputs of $\textnormal{ETTm}_1$ with the horizon of 24 for LRU/BLUR. From the plot, we can observe that BLUR evidently exhibits better prediction capability (closer to ground truth), compared to LRU. We refer interested readers to Appendix~\ref{additional_results} for much more comparison, including BRNN/BLSTM and the more recently developed Mamba-type models.

\vspace{5pt}

\begin{table}[ht!]
\centering
\caption{Multivariate long sequence time-series forecasting results with MAE on three datasets (Four cases). Full results involving other baselines are in Appendix~\ref{additional_results}. H: horizon, B: BLUR, L: LRU, S: S4, and I: Informer.}
\begin{tabular}{|c|c|c|c|c|c|}
\hline
\textbf{Data} & \textbf{H} & \textbf{B} & \textbf{L} &\textbf{S} &\textbf{I}\\
\hline
\multirow{5}{*}{\rotatebox[origin=c]{90}{\textbf{ETTh$_1$}}} 
 & 24   &  \textbf{0.300} &  0.319  & 0.542 &0.549\\
 & 48   &  0.484 & \textbf{0.383}  &0.615 & 0.625\\
 & 168  & \textbf{ 0.508} &  0.523  &0.779 & 0.752\\
 & 336  & \textbf{0.336}  & 0.696   &0.910 & 0.873\\
 & 720  &  \textbf{0.536} & 0.880  &0.842 & 0.896\\
\hline
\multirow{5}{*}{\rotatebox[origin=c]{90}{\textbf{ETTh$_2$}}} 
 & 24   & \textbf{0.108} & 0.230  &0.736 & 0.665\\
 & 48   & \textbf{0.214}  & 0.275   &0.867 & 1.001\\
 & 168  & \textbf{0.803} & 1.030  &1.255 &1.515\\
 & 336  & \textbf{0.853} & 1.160   & 1.128 & 1.340\\
 & 720  & \textbf{0.640} &  2.050  &1.340 & 1.473\\
\hline
\multirow{5}{*}{\rotatebox[origin=c]{90}{\textbf{ETTm$_1$}}} 
 & 24   &  \textbf{0.121} &0.209   &0.487 & 0.369\\
 & 48   &  \textbf{0.220} &  0.288  & 0.565 & 0.503\\
 & 96   &  \textbf{0.337}  &  0.339  &0.649 & 0.614\\
 & 288   & 0.550 &  \textbf{0.465}  &0.674 & 0.786\\
 & 672  &  \textbf{0.644} & 0.672    &0.709 & 0.926\\
\hline
\multirow{5}{*}{\rotatebox[origin=c]{90}{\textbf{Weather}}}  
 & 24   & \textbf{0.186} & 0.238  &0.384 & 0.381\\ 
 & 48   & \textbf{0.232} & 0.290   &0.444 & 0.459\\
 & 168  &   \textbf{0.332}  & 0.413  &0.527 & 0.567\\
 & 336  &   \textbf{0.362} & 0.472  &0.539 & 0.620\\
 & 720  &   \textbf{0.441} & 0.528  &0.578 & 0.731\\
\hline
\multicolumn{2}{|c|}{Count} & 18 & 2 & 0 & 0\\
\hline
\end{tabular}
\label{table:timeseries}
\end{table}


\begin{figure}[h!]
    \centering
    \includegraphics[width=0.9\linewidth]{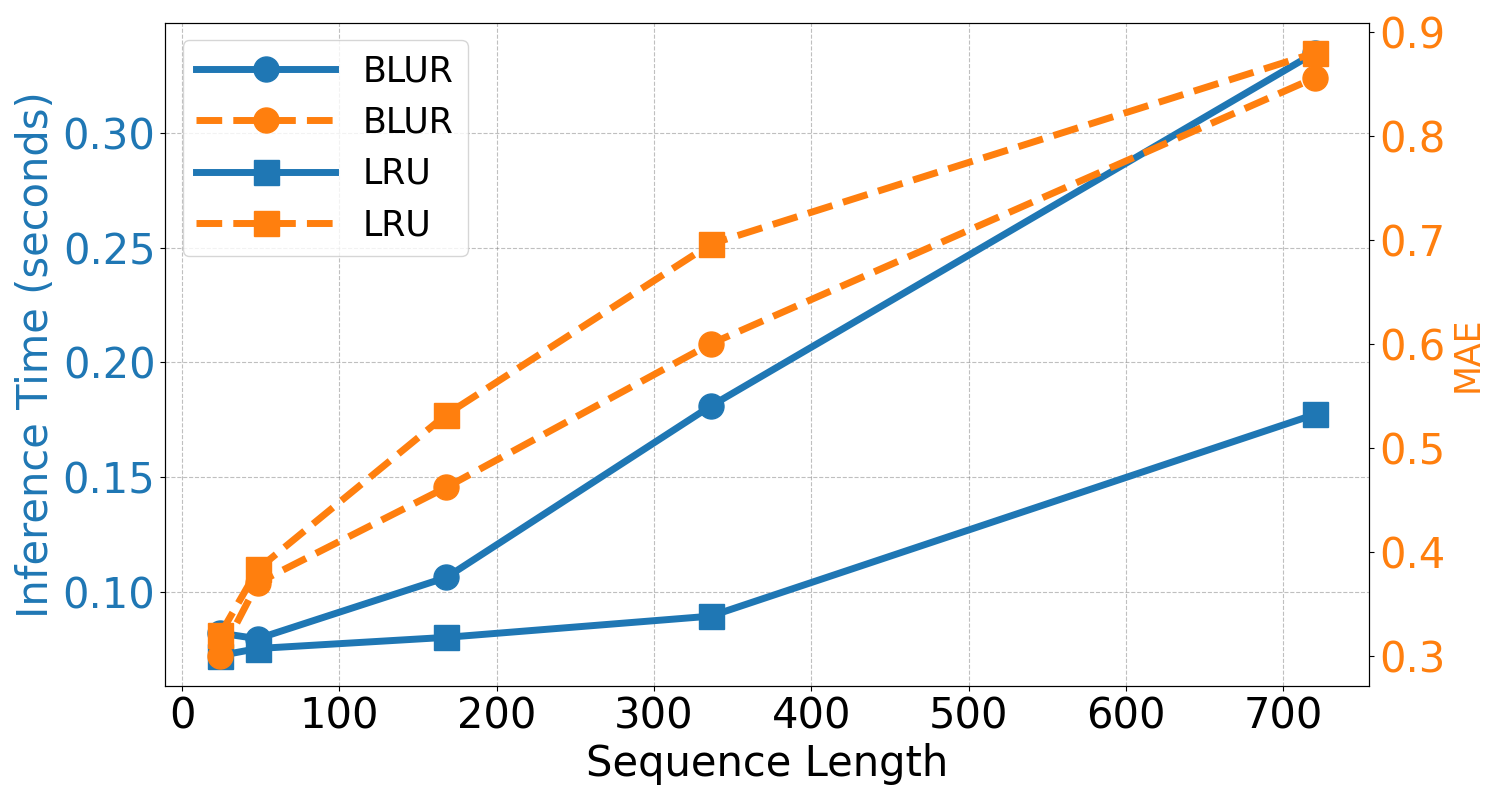}
    \caption{Inference time costs and test MAE along the sequence length with ETTh$_1$.}
    \label{fig:inferece_speed}
    \vspace{10pt} 
\end{figure}

\textbf{Inference time and prediction horizon.}
We also inspect the inference time and the MAE score in Figure~\ref{fig:inferece_speed} along with the prediction length for LRU and BLUR. It is anticipated that both inference time and MAE will increase if the horizon enlarges. The plot also implies that BLUR gains performance improvement, with the compromise of inference time, compared to LRU. To extensively investigate the impact of prediction horizon on the metrics for different types of models, Figure~\ref{fig:horizon_plot} shows the results for linear recurrence (BLUR, LRU, S4), nonlinear recurrence (LSTMa and BLSTM), CNN (TCN), and transformer (Autoformer). One immediate observation is that the longer horizon results can result in larger MSE and MAE due to the negative effect of compounding error. In general, BLUE still surpasses baselines, underscoring the significance of LRU-type models in timeseries prediction tasks.

\begin{figure}[ht!]
\centering
\includegraphics[width=1\linewidth]{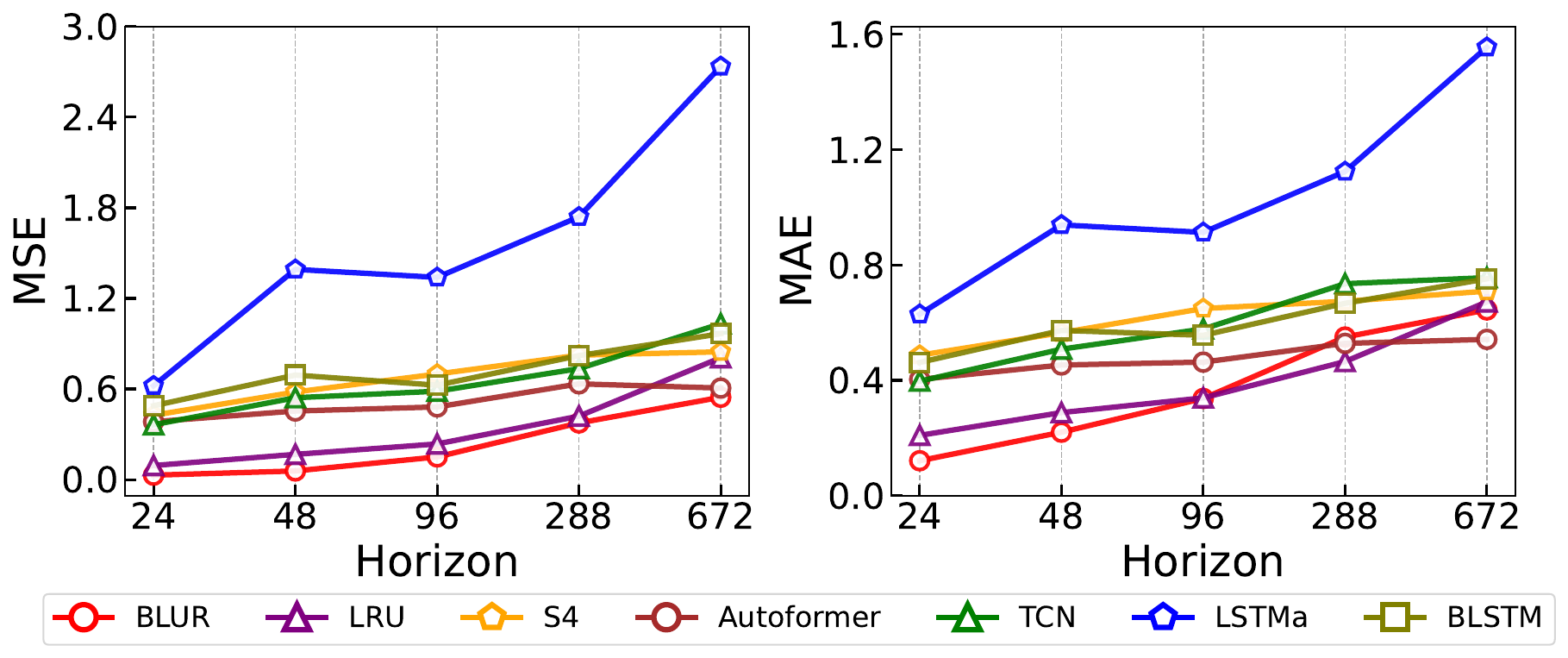}
\caption{MSE and MAE comparison for different models along different horizons for ETTm$_1$ dataset.}
\label{fig:horizon_plot}
\end{figure}

\section{Conclusion}
\label{conclusion}
In this work, we propose BLUR, a novel bidirectional variant of the recently proposed LRU model that leads to fast parallel training in modeling long-term dependencies. We theoretically analyze the stability of BLUR and establish its approximation properties of capturing temporal patterns in sequences. To validate the proposed method, we have used multiple popular benchmark datasets in sequential images and text and timeseries, comparing results between BLUR and baselines. Numerical results empirically show that BLUR has superiority in computational performance and efficiency, over baselines. Particularly, its impressive parallel training capability inherited from LRU demonstrates its great potential when deploying it on long sequence data. We have also found that LRU-type models perform significantly better than transformers on timeseries predictions. Beyond the current results, a few future research directions include investigating the decay rate of the approximation error and testing BLUR in modeling long-term dependencies in other domains, such as NLP and biology.

\nocite{langley00}
\bibliographystyle{elsarticle-num} 
\bibliography{blur}

\newpage
\appendix
\onecolumn
\section{Appendix}
\subsection{Reference Summary for All Baseline Methods}\label{references}
As in this work there are numerous baselines (totally 36) adopted in the main content and appendix, we summarize the references for all baseline methods for readers to understand easily. Please see the following table for more details.
\begin{table*}[htp]
\caption{References for all baselines.}
\begin{center}
\begin{threeparttable}
\begin{tabular}{c c}
    \toprule
    \textbf{Method} & \textbf{Reference} \\ \midrule
    TCN& \cite{bai2018empirical} \\
    Transformer&\cite{vaswani2017attention}\\
    Reformer & \cite{kitaev2020reformer} \\
    LogTrans & \cite{nie2022logtrans}\\
    Informer & \cite{zhou2021informer}\\
    Informer+ & \cite{zhou2021informer}\\
    Fedformer & \cite{zhou2022fedformer}\\
    Autoformer & \cite{wu2021autoformer}\\
      LRU   &              \cite{orvieto2023resurrecting}     \\
      SOFTS & \cite{han2024softs}   \\
      SiMBA & \cite{patro2024simba}\\
      S-Mamba & \cite{wang2025mamba} \\
      PatchTST & \cite{nie2022time}\\
      Pyraformer & \cite{liu2022pyraformer}\\
      SegRNN & \cite{lin2023segrnn}\\
      TimesNet & \cite{wu2022timesnet} \\
      MICN & \cite{wang2023micn}\\
      S4 & \cite{gu2021efficiently} \\
      S5 & \cite{orvieto2023resurrecting}\\
      Crossformer & \cite{zhang2023crossformer}\\
      iTransformer & \cite{liu2023itransformer}\\
      TiDE & \cite{das2023long} \\
      RLinear & \cite{li2023revisiting} \\
      CoST & \cite{woo2022cost} \\
      MoCO & \cite{he2020momentum} \\
      TS2Vec & \cite{yue2022ts2vec}\\
      LSTnet & \cite{lai2018modeling}\\
      LSTMa & \cite{bahdanau2014neural}\\
      BRNN & \cite{schuster1997bidirectional}\\
      BLSTM & \cite{huang2015bidirectional}\\
      TNC & \cite{tonekaboni2021unsupervised}\\
    CMMamba & \cite{li2024cmmamba} \\
    Bi-Mamba+ & \cite{liang2024bi}\\
    UMambaTSF & \cite{wu2024umambatsf}\\
      TimeMachine   &              \cite{ahamed2024timemachine}     \\
      SST &\cite{xu2024integrating}\\
      \bottomrule
      
\end{tabular}
\begin{tablenotes}

\end{tablenotes}
\end{threeparttable}
\end{center}
\label{table:references}
\end{table*}

\subsection{Additional Related Works}\label{additional_related_works}
Another line of work is linear Transformers~\cite{katharopoulos2020transformers,schlag2021linear} where an autoregressive Transformer was constructed with a linear attention mechanism, which exhibits a close relationship with linear RNNs. Specifically, these Transformers can be reformulated as RNNs during autoregressive decoding, whose updates are regarded approximately as a special case of element-wise linear recurrence. Traditional RNNs are also deficient in modeling long-term dependencies due to the vanishing gradient issue. Apart from the gating mechanisms, another remedy is to regularize or initialize the eigenvalues of the recurrent weight matrix with identity~\cite{le2015simple} or unitary~\cite{arjovsky2016unitary} matrices. The authors in~\cite{orvieto2023resurrecting} resorted to randomized matrices to initialize eigenvalues to be close to one. To provide a clearer illustration, we present our architecture diagram in Figure~\ref{fig:blur_architecture}. The BLUR architecture builds upon the state space model from LRU, and we incorporate bidirectional techniques to enhance performance.


\subsection{Full Qualitative Comparison Table}\label{full_comparison_table}
We also have the full qualitative table for the comparison, shown in Table~\ref{table:comparison_full}, including other methods and the extra memory column. 
\begin{table*}[ht]
\caption{Qualitative comparison among approaches.}

\begin{center}
\resizebox{\textwidth}{!}
{%
\begin{threeparttable}
\begin{tabular}{c c c c}
    \toprule
    \textbf{Method} & \textbf{Time} &\textbf{Memory}& \textbf{Mechanism}\\ \midrule
    RNN~\cite{hewamalage2021recurrent} & $\mathcal{O}(N)$& $\mathcal{O}(N)$ & Nonlinear recurrence \\
    TCN~\cite{bai2018empirical} & $\mathcal{O}(N)$ & $\mathcal{O}(N)$ & 1D convolutional network + causal convolutions\\
    Transformer~\cite{vaswani2017attention}&$\mathcal{O}(N^2)$& $\mathcal{O}(N^2)$ &Self-attention\\
    TFT$^*$~\cite{lim2021temporal} & N/A & N/A & Variable selection + static covariate\\
    Reformer~\cite{kitaev2020reformer} & $\mathcal{O}(N\textnormal{log}N)$ & $\mathcal{O}(N\textnormal{log}N)$ & LSH attention \\
    N-BEATS$^*$~\cite{oreshkin2019n} &N/A & N/A& Doubly residual stacking + ensemble \\
    LogTrans~\cite{nie2022logtrans} & $\mathcal{O}(N\textnormal{log}N)$ & $\mathcal{O}(N^2)$ & CNN + attention\\
    Informer~\cite{zhou2021informer} & $\mathcal{O}(N\textnormal{log}N)$ & $\mathcal{O}(N\textnormal{log}N)$ & Probsparse self-attention\\
    Fedformer~\cite{zhou2022fedformer} & $\mathcal{O}(N)$ & $\mathcal{O}(N)$ & FEB + FEA + MOE Decomp\\
    Autoformer~\cite{wu2021autoformer} & $\mathcal{O}(N\textnormal{log}N)$ & $\mathcal{O}(N\textnormal{log}N)$ & series decomposition block + auto-correlation mechanism\\
    N-HiTS$^*$~\cite{challu2023nhits} & $\mathcal{O}(N\frac{1-\varsigma^B}{1-\varsigma})$ & $\mathcal{O}(N\frac{1-\varsigma^B}{1-\varsigma})$ & Doubly residual stacking principle \\
    FiLM~\cite{zhou2022film} & $\mathcal{O}(N)$& $\mathcal{O}(N)$ & Legendre projection unit + frequency enhanced layer\\
    DLinear~\cite{zeng2023transformers} &$\mathcal{O}(N)$ &$\mathcal{O}(N)$& Decomposition scheme + linear layer\\
    S4~\cite{gu2021efficiently} & $\mathcal{\tilde{O}}(N)$ &$\mathcal{\tilde{O}}(N)$ & Linear recurrence\\
      LRU~\cite{orvieto2023resurrecting}   &              $\mathcal{O}(N)$ & $\mathcal{O}(N)$&Dia-linear recurrence     \\
      SOFTS~\cite{han2024softs} & $\mathcal{O}(N)$ & $\mathcal{O}(N)$ &  Star aggregate-redistribute   \\
      SiMBA~\cite{patro2024simba} & $\mathcal{O}(N\textnormal{log}N)$ &$\mathcal{O}(N\textnormal{log}N)$& Mamba + EinFFT\\
      S-Mamba~\cite{wang2025mamba} & $\mathcal{O}(N\textnormal{log}N)$ &$\mathcal{O}(N\textnormal{log}N)$& Mamba + MLP \\
      \textbf{BLUR}   &     $\mathcal{O}(N)$ & $\mathcal{O}(N)$ & Dia-bi-linear recurrence               \\ 
      \bottomrule
      
\end{tabular}
\begin{tablenotes}
Time: training time complexity,  LSH: locality sensible hashing, Dia-linear: diagonal linear, Dia-bi-linear: diagonal bidirectional linear. FEA: frequency enhanced attention. FEB: frequency enhanced block. MOE Decomp: mixture of expert decomposition. * means the proposed model is more for interpretability. $B$: the number of blocks in the model, $\varsigma$: the expressivity ratio. EinFFT: Einstein Fast Fourier Transform.
\end{tablenotes}
\end{threeparttable}%
}
\end{center}
\label{table:comparison_full}
\end{table*}
\subsection{Additional Proof for Theorem~\ref{theorem_1}}\label{additional_proof}
In this subsection, we present the detailed proof for Theorem~\ref{theorem_1}. Note that the major proof ideas and techniques are adapted from~\cite{orvieto2023universality}. We remark that the adaptation is still non-trivial as the architecture has become bidirectional LRUs. We restate the theorem here for completeness. Before that, we provide the detailed steps for the proof to add clarify and make it clear.
\begin{itemize}
    \item For the first step, we show the linear RNNs are able to memorize the inputs or encode efficiently the inputs to the latent space, which requires Lemma~\ref{lemma_1}
    \item Then, we show the approximation property of MLP based on the Theorem~\ref{theorem_2},  as the MLP layer brings nonlinearity for the BLUR to make the model more expressive.
    \item We then construct the forward and backward Vandermonde matrices to represent the updates of the bidirectional recurrence in BLUR. Additionally, we establish the linear reconstruction mappings induced by the Moore-Penrose inverses of both forward and backward Vandermonde matrices, resulting in the function extension with the new Barron constant in Theorem~\ref{theorem_3} and the assistance in constructing Lemma~\ref{lemma_2}.
    \item We show a result in Lemma~\ref{lemma_2} that is adapted from Theorem~\ref{theorem_2}  to present the relationship between the number of neurons and the approximate accuracy of the nonlinear layer.
    \item The last step is to combine the aforementioned results. We first use Lemma~\ref{lemma_1} to construct the forward and backward LRUs that can memorize the inputs. Then, based on Lemma~\ref{lemma_2} and Theorem~\ref{theorem_3} (as well as its backward version), we can obtain there exists a function extension with a new Barron constant, which is considered a nonlinear layer that takes inputs from both forward and backward LRUs. We then let the function extension be one hidden-layer MLP and apply Theorem~\ref{theorem_2} to complete the proof.
\end{itemize}

\textbf{Theorem~\ref{theorem_1}:} 
Suppose that the inputs $\bm{v}=(\bm{v}_i)_{i=1}^N\in\mathcal{V}\subseteq\mathbb{R}^{d\times N}$ are bounded, i.e., $\|\bm{v}_i\|<\infty$ for all $i\in[N]$. Denote by $\textnormal{dim}(\mathcal{V})$ the vector-space dimension of $\mathcal{V}$. Let $\mathcal{L}_1$ and $\mathcal{L}_2$ be LRU with either real or complex eigenvalues and with width $n\geq \textnormal{dim}(\mathcal{V})$. Let the nonlinear mapping $\mathcal{E}$ be parameterized by an MLP with width $D\geq \mathcal{O}(nN/\varepsilon^2)$. Hence, $\hat{\mathcal{F}}=\mathcal{L}_3\circ\mathcal{G}\circ\mathcal{H}\circ(\mathcal{L}_1+\mathcal{L}_2)\circ\mathcal{E}$ approximates pointwise $\mathcal{F}$ with error $\varepsilon$:
    \begin{equation}
        \textnormal{sup}_{\bm{v}\in\mathcal{V}}\|\hat{\mathcal{F}}(\bm{v})-\mathcal{F}(\bm{v})\|\leq \varepsilon.
    \end{equation}

To show the proof, we first introduce a few existing results from~\cite{orvieto2023universality}. The first one is regarding the capabilities of linear RNNs, which tells that such models are able to memorize inputs, which sets the foundation of the proof.
\begin{lemma}\label{lemma_1}
    Denote by $\mathcal{U}\subseteq \mathbb{R}^{m\times N}$ the encoded input set. If the hidden dimension scales with $mN$, the linear RNN model can recover the encoded inputs $(\bm{u}_i)_{i=1}^k$ from the hidden states $\bm{h}_k$, for all $k=1,2,...,N$, without any information loss. If $\textnormal{dim}(\mathcal{U})=P<mN$, i.e., $\mathcal{U}$ is linearly compressible, then the hidden dimension can be reduced to $\mathcal{O}(P)$.
\end{lemma}
We next present a well-known result on the approximation property of MLP based on~\cite{barron1993universal}.
\begin{theorem}\label{theorem_2}
    Consider $\mathcal{G}(\bm{h})$ parameterized by a one hidden-layer MLP (with $D$ hidden neurons): $\mathcal{G}(\bm{h})=\sum_{k=1}^D\tilde{c}_k\sigma(\langle\tilde{a}_k,\bm{h}\rangle+\tilde{b}_k)+\tilde{c}_0$, where $\sigma$ is any sigmoidal function. Let $f:\mathbb{R}^p\to\mathbb{R}$ be continuous with Barron constant $\mathcal{C}_f$. If $D\geq 2nr^2\mathcal{C}_f^2\varepsilon^{-2}$, where $r:=\textnormal{sup}_j\{\bm{h}_j|j\in[1,n]\}$, then there exist parameters such that $\textnormal{sup}_{\|\bm{h}\|\leq \sqrt{n}r}|f(\bm{h})-\mathcal{G}(\bm{h})|\leq \varepsilon$.
\end{theorem}
The detailed proof of Theorem~\ref{theorem_2} can be found in~\cite{barron1993universal}. In what follows, we establish a family of special matrices, called Vandermonde matrix that can represent the updates of both forward and backward recurrence in Eq.~\ref{eq_15} and Eq.~\ref{eq_17}. For simplicity, we set $d=m=1$, the encoder as the identity, and $\mathbf{B}=[1,1,...,1]^\top$. Then Eq.~\ref{eq_15} can be rewritten as
\begin{equation}
    \bm{h}_k=\begin{bmatrix}
\lambda_{f,1}^{k-1} & \lambda_{f,1}^{k-2}  & \cdots & \lambda_{f,1} & 1 \\ 
\lambda_{f,2}^{k-1} & \lambda_{f,2}^{k-2}  & \cdots & \lambda_{f,2} & 1 \\
\vdots & \vdots & \vdots & \ddots &\vdots \\
\lambda_{f,n}^{k-1} & \lambda_{f,n}^{k-2}  & \cdots & \lambda_{f,n} & 1 \\ 
\end{bmatrix}
\begin{bmatrix}
    \bm{u}_1\\
    \bm{u}_2\\
    \vdots\\
    \bm{u}_k
\end{bmatrix}=\mathbf{V}_{f,k}\bm{u}^\top_{1:k},
\end{equation}
where $\bm{u}_{1:k}=\bm{v}_{1:k}=(\bm{v}_i)_{i=1}^k\in\mathbb{R}^{1\times k}$, and $\mathbf{V}_{f,k}$ is a \textit{forward Vandermonde matrix}. When $n\geq k$, we can use the Moore-Penrose inverse of $\mathbf{V}_{f,k}$ to reconstruct $\bm{u}_{1:k}$ such that
\begin{equation}
    \bm{v}_{1:k}^\top=\bm{u}_{1:k}^\top=\mathbf{V}_{f,k}^+\bm{h}_k.
\end{equation}
Similarly, Eq.~\ref{eq_17} can be rewritten as
\begin{equation}
    \bm{h}_1=\begin{bmatrix}
\lambda_{b,1}^{k-1} & \lambda_{b,1}^{k-2}  & \cdots & \lambda_{b,1} & 1 \\ 
\lambda_{b,2}^{k-1} & \lambda_{b,2}^{k-2}  & \cdots & \lambda_{b,2} & 1 \\
\vdots & \vdots & \vdots & \ddots &\vdots \\
\lambda_{b,n}^{k-1} & \lambda_{b,n}^{k-2}  & \cdots & \lambda_{b,n} & 1 \\ 
\end{bmatrix}
\begin{bmatrix}
    \bm{u}_k\\
    \bm{u}_{k-1}\\
    \vdots\\
    \bm{u}_1
\end{bmatrix}=\mathbf{V}_{b,k}\bm{u}^\top_{k:1}.
\end{equation}
$\mathbf{V}_{b,k}$ is a \textit{backward Vandermonde matrix}.
By using the Moore-Penrose inverse, we have
\begin{equation}
    \bm{v}_{k:1}^\top=\bm{u}_{k:1}^\top=\mathbf{V}_{b,k}^+\bm{h}_1.
\end{equation}
Lemma~\ref{lemma_1} implies a worse-case setting that requires the RNN hidden state dimension $n$ to be of size proportional
to the input sequence $N$. Generally speaking, it is impossible to store $\mathcal{O}(N)$ real numbers in a vector whose size is smaller than $\mathcal{O}(N)$ through a linear Vandermonde projection, unless the inputs are sparse. This will require another assumption to parameterize the input sequence with some basis functions, as done in~\cite{orvieto2023universality}. In this study, we only consider the non-sparse setting such that the linear reconstruction mappings for the forward and backward passes are exactly $\mathbf{V}_{f,k}$ and $\mathbf{V}_{b,k}$. Without loss of generality, we also consider real/imaginary inputs representation in $\mathbb{R}^{2n\times n}$.
Typically, $\hat{\mathcal{F}}$ should be time-independent such that we leverage a recent result on harmonic analysis from~\cite{tlas2022bump} to construct an approximate domain relaxation and operate with $\mathcal{G}$ (which is represented by a single MLP). This suggests a result for the Barron complexity for function sequences that involve $N$ time steps. Hence, we can obtain an input dimension that determines the function $\mathcal{G}$ should be implemented. In addition to the Barron constant we have defined in the above, $\mathcal{C}_f$, we denote by $\mathcal{C}'_f=\int_{\mathbb{R}^p}|\tilde{f}(\omega)|d\omega$ the Fourier norm.
In the sequel, we formally state the result for the combinations of Barron maps.
\begin{theorem}\label{theorem_3}
    Let $f: [1,2,...,N]\times\mathbb{R}^p\to\mathbb{R}$ be a function mapping such that each $f(k,\cdot)$ is Barron with constant $\mathcal{C}_{f_k}$ and has Fourier norm $\mathcal{C}'_{f_k}$. There exists an extension $\hat{f}:\mathbb{R}^{p+1}\to\mathbb{R}$ of $f$, where $\hat{f}$ is Barron with constant $\hat{\mathcal{C}}\leq \mathcal{O}(\sum_{k=1}^N\mathcal{C}_{f_k}+\mathcal{C}'_{f_k})=\mathcal{O}(N)$ and $\hat{f}(x,k)=f(x,k), \forall x\in\mathbb{R}^p, \forall k\in[N]$.
\end{theorem}
The proof for Theorem~\ref{theorem_3} follows similarly from the proof for Theorem 4 in~\cite{orvieto2023universality}. However, this only applies to the forward pass, and we need to ensure a similar conclusion can apply to the backward pass in BLUR as well. Since the backward LRU has the same architecture and is initialized independently, we can construct the function mapping for the backward LRU that complies with the conclusion in Theorem~\ref{theorem_3}. In the next, we focus first on obtaining the approximation of $\mathcal{F}_k$ by using the MLP, at a fixed time step $k$ from the LRU. Denote by $\Theta_{f,k}$ the functional of the linear reconstruction mapping induced by the Moore-Penrose inverse of the forward Vandermonde matrix $\mathbf{V}_{f,k}^+$ such that $\mathcal{G}(\bm{h}_k)=\mathcal{F}_k(\Theta_f(\bm{h}_k)):=f_k(\bm{h}_k)$. Similarly, we can also define as $\Theta_{b,k}$ the functional of the linear reconstruction mapping induced by the Moore-Penrose inverse of the backward Vandermonde matrix $\mathbf{V}_{b,k}^+$.
With these in hand, we are now ready to state the following result.
\begin{lemma}\label{lemma_2}
    Suppose that the MLP layer $\mathcal{G}$ is sufficient to approximate $f_k:=\mathcal{F}_k\circ(\Theta_{f,k}+\Theta_{b,k})$ at level $\varepsilon$ from BLUR $\hat{\bm{h}}_k$. Then the number of hidden neurons $D$ in $\mathcal{G}$ satisfies the following relationship
    \begin{equation}
        D\leq 4nr^2\mathcal{C}^2_{\mathcal{F}_k}\|\Theta_{f,k}+\Theta_{b,k}\|_2^2s^3\varepsilon^{-2},
    \end{equation}
where $\|\hat{\bm{h}}_k\|_2\leq \sqrt{n}r$, for all $k\in[N]$, and $\mathcal{C}_{\mathcal{F}_k}$ is the Barron constant of $\mathcal{F}_k$.
\end{lemma}
\begin{proof}
    Applying similar proof steps of Proposition 2 in~\cite{orvieto2023universality} with a different factor $\|\Theta_{f,k}+\Theta_{b,k}\|_2$ completes the proof.
\end{proof}
The conclusion of Lemma~\ref{lemma_2} can be adapted from Theorem~\ref{theorem_2}. However, in Theorem~\ref{theorem_2}, the MLP directly applies to the raw inputs $\bm{v}$ such that $\hat{y}_k=\mathcal{G}((\bm{v}_i)_{i=1}^k)$, which results in $4nr^2\mathcal{C}^2_{\mathcal{F}_k}s^3\varepsilon^{-2}$. Instead, in Lemma~\ref{lemma_2}, the operation is on hidden states $\hat{\bm{h}}_k$ such that there is an extra factor $\|\Theta_{f,k}+\Theta_{b,k}\|_2$.
We are ready to give the detailed proof for Theorem~\ref{theorem_1}. 
\begin{proof}(Proof of Theorem~\ref{theorem_1}): Let $f_k:=\mathcal{F}_k\circ(\Theta_{f,k}+\Theta_{b,k}):\mathbb{R}^{2n}\to\mathbb{R}^s$. Based on Lemma~\ref{lemma_2}, we can know that $\mathcal{C}_{f_k}=\|(\Theta_{f,k}+\Theta_{b,k})\mathcal{C}_{\mathcal{F}_k}\|$. $\mathcal{C}'_{f_k}$ scales similarly. Due to Lemma~\ref{lemma_1}, we can obtain that $f_k(\bm{h}_k)=\mathcal{F}(\bm{v}_{1:k})$. With this, we then apply Theorem~\ref{theorem_3} and its backward version to the sequence of $f_k$ (one for each time step): there exists an interpolation function $\hat{f}$ with Barron constant $\mathcal{O}(\sum_{k=1}^N\mathcal{C}_{f_k}+\mathcal{C}'_{f_k})=\mathcal{O}(N)$. Letting this be the specific function the one hidden-layer MLP has to approximate and applying Theorem~\ref{theorem_2} yields the desirable result.
\end{proof}
\subsection{Model Architecture and Hyperparameter Setup}\label{models}
We provide the details on the proposed BLUR in this subsection. Following the same architecture of LRU, BLUR has four layers of LRUs in each direction. Specifically, the forward block of BLUR has four layers of LRUS and the same as the backward block. The latent size of the recurrent unit is 128. We also have the linear embedding layer to encode the raw input to the embedding space with a dimension being 256. The time pool operation in Figure~\ref{fig:blur} is the mean. To control the eigenvalue of the diagonal matrix $\mathbf{\Lambda}$, $e_{min}$ and $e_{max}$ are 0 and 1 respectively. We also have the batch normalization between different components in BLUR to boost its empirical performance. The batch size we have set is 64. In training, for the time-series dataset, the number of epochs is 8, which is the same as in~\cite{zhou2021informer}. For sequential images and text datasets, the number of epochs ranges from 40 to 180. We also have the decay learning rate set as in~\cite{orvieto2023resurrecting}. The dropout rate is 0.1 to be consistent with that for the original LRU. We summarize all model and hyperparameter settings as follows.
\begin{table}[htp]
\caption{Model and hyperparameter setup for BLUR}
\begin{center}
\begin{threeparttable}
\begin{tabular}{c c}
    \toprule
    Hyper. & Value\\ \midrule
      \# of layers   &        4\\
      Latent size & 128\\
     Embedding dimension & 256\\
      $[e_{min},e_{max}]$ & [0,1]\\
      Normalization   &       Batch\\ 
      Batch size & 64\\
      \# of epochs for timeseries & 8\\
      \# of epochs for images and text & [40, 180]\\
      Base learning rate & 0.001\\
      Minimum learning rate & 0.0000001\\
      Decay rate of learning rate & 0.7\\
      Dropout rate & 0.1\\
      Weight decay & 0.05\\
      \bottomrule
      
\end{tabular}
\end{threeparttable}
\end{center}
\label{table:seq_images_full}
\end{table}
\subsection{Detail about Datasets}\label{datasets}
We provide additional detail for each dataset in both tasks.

\textbf{Sequential CIFAR.} The task is to look at a $32\times 32$ CIFAR-10 image and predict the class of the image. There are 10 classes in the dataset, including airplane, automobile, bird, cat, deer, dog, frog, horse, ship, and truck. This dataset is a slightly different version of the popular CIFAR 10 dataset~\cite{recht2018cifar}. The size of sequential CIFAR is 60K.

\textbf{ListOps.} This dataset consists of summary operations on lists of single-digit integers written in prefix notation. The full sequence has a corresponding solution, which is also a single-digit integer, thus making it a ten-way balanced classification problem. Please check~\cite{nangia2018listops} for more detail. The size of ListOps is 100K.

\textbf{Text.} This is the popular IMDB dataset~\cite{yenter2017deep} that comprises 50K movie reviews for natural language processing or text analytics. It is a binary sentiment classification problem. 

\textbf{Sequential MNIST~\cite{vorontsov2017orthogonality}.} This is a standard benchmark task for timeseries classification where each input has sequences of 784 pixel values created by unrolling the MNIST digits, pixel by pixel. The size of the dataset is 70K, including training and testing datasets.

\textbf{Permutated sequential MNIST~\cite{le2015simple}.} The psMNIST dataset is a different version of the sequential MNIST dataset with "a fixed random permutation of the pixels of the digits". It is a transformation of the MNIST dataset to evaluate sequence models. The size is the same as that of sequential MNIST.

\textbf{ETTh$_1$.} It is an electricity dataset with a 2-year range from 2016/07 - 2018/07. It contains 17520 data points sampled hourly. Each data point includes 8 features, i.e., the date of the sample, the target "oil temperature", and six different types of external power load features. The train/val/test is 12/4/4 months.

\textbf{ETTh$_2$.} This dataset has the same information as in \textbf{ETTh$_1$}, while collected at a different electricity station. The train/val/test is 12/4/4 months

\textbf{ETTm$_1$.} It is an electricity dataset with a 2-year range from 2016/07 - 2018/07. It contains 70080 data points sampled every 15 minutes. Each data point includes 8 features, i.e., the date of the sample, the target "oil temperature", and six different types of external power load features. The collection happened in the same place as \textbf{ETTh$_1$}. It can be regarded as a minute-version of \textbf{ETTh$_1$}. The train/val/test is 12/4/4 months.

\textbf{ECL.} This is a dataset containing electricity consumption of 370 points/clients. There are a total of 140256 data samples. The values are in kW of each 15 min. Please see~\cite{zhou2021informer} for more detail. The train/val/test is 15/3/4 months.

\textbf{Weather.} The weather data was collected by NOAA. The dataset contains local climatological data for almost 1600 U.S. locations, four years from 2010 to 2013. The sampling rate is 1 hour. Each data sample comprises the target value "wet bulb" and 11 climate features. The train/val/test is 28/10/10 months.

\subsection{Parallel Scan}\label{parallel_scan}
In traditional nonlinear recurrent models such as RNNs, LSTMs, and GRUs, backpropagation through time (BPTT) is the method to obtain the gradients and then update the parameters. However, BPTT can be extremely time-consuming if the sequence length is large. The emergence of transformers replaced these models due to their parallel training capabilities while suffering the quadratic time complexity, as discussed before. Recently, the structured state-space models have invoked renewed interest in recurrent models by replacing the nonlinear recurrence with the linear recurrence. Many new linear recurrent models have also been proposed to avoid the BPTT issue thanks to the parallel scan algorithm, which is introduced briefly as follows.

The \textit{Parallel Scan} algorithm~\cite{heinsen2023parallelization} is an approach for parallel computation for computing $N$ prefix computations from $N$ sequential data points through an associate operator $\oplus$ such as $+$ and $\times$. The parallel scan algorithm is able to efficiently computes $\{\oplus_{i=1}^k\bm{u}_i\}_{k=1}^N$ from $\{\bm{u}_k\}_{k=1}^N$. Particularly, we can directly apply the parallel scan algorithm to efficiently compute a popular family of functions: $h_k=a_kh_{k-1}+b_k$, where $h_k, a_k, b_k\in\mathbb{R}$ and $h_0=b_0$. This method takes as input $a_1, a_2, ..., a_N$ and $b_1, b_2, ..., b_N$ and computes $h_1, h_2, ..., h_N$ via parallel scans.

\subsection{Additional Results}\label{additional_results}
For the timeseries prediction tasks, in this work, we try our best to follow the same setting as in~\cite{zhou2021informer}. The prediction tasks involving different horizons are shown in Figure~\ref{fig:prediction_tasks}. To comply with their setting, the input and output lengths are all set equal to each other. However, they can be different in other scenarios.
\begin{figure}
    \centering
    \includegraphics[width=0.5\linewidth]{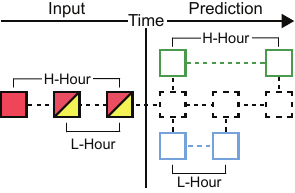}
    \caption{Schematic diagram of prediction tasks for different input/output lengths.}
    \label{fig:prediction_tasks}
\end{figure}

\textbf{Sequential images and text datasets.} To compare BLUR additionally to other existing models, including S4 and S4D-Legs~\cite{gu2021efficiently} on the sequential images and text datasets, we add two more rows to Table~\ref{table:seq_images}, which has now become Table~\ref{table:seq_images_full}. However, we would like to note that the additional results in Table~\ref{table:seq_images_full} are directly attained from their paper without reproducing them. One can immediately observe that S4 has the highest accuracy on the sequential CIFAR dataset. However, it should be acknowledged that the gap between LRU and S4 is significant, while BLUR can reduce the gap with the bidirectional mechanism. Also, BLUR still maintains superior on the Text and permuted sequential MNIST datasets as in Table~\ref{table:seq_images}. The difference can partially be attributed to different running settings, but presumably, this distinction should be marginal.
\begin{table}[htp]
\caption{Performance comparison between LRU and BLUR on sequential image and text datasets.}
\begin{center}
\begin{threeparttable}
\begin{tabular}{c c c c c c}
    \toprule
    \textbf{Method} & \textbf{sC} & \textbf{LO} & \textbf{TE} & \textbf{sM} & \textbf{psM}\\ \midrule
      LRU   &        89.0            & 54.8    & 86.7& 99.3& \textbf{98.1}\\
      S4* & \textbf{91.1}& 59.6& 86.8& 90.9&94.2\\
      S4D-Legs*~\cite{gupta2022diagonal} & 89.9& \textbf{60.5}& 86.2& 89.5&93.1\\
      S5* & 88.8& 58.5& 86.2& 88.9&95.7\\
      \textbf{BLUR}   &90.0& 54.6                     &\textbf{87.3}    &\textbf{99.5} & \textbf{98.1}       \\ 
      \bottomrule
      
\end{tabular}
\begin{tablenotes}
sC: sCIFAR, LO: ListOps, TE: Text, sM: sMNIST, psM: psMNIST, * means the results are taken from the paper directly.
\end{tablenotes}
\end{threeparttable}
\end{center}
\label{table:seq_images_full}
\end{table}

\textbf{Timeseries dataset.} We include full tables to show the comprehensive evaluation of LRU and BLUR on different timeseries datasets with more baselines. Tables~\ref{table:timeseries_full_table1} - \ref{table:timeseries_full_table4} still show that BLUR significantly outperforms other baseline methods, resulting in a competitive sequencing model rival in modeling long-term dependencies. Also, the bidirectional mechanism enhances the model performance on top of LRU. We also show the comparison between BLUR and other two popular bidirectional mechanisms, BRNN and BLSTM, in Table~\ref{table:timeseries_full_table5}. Likewise, BLUR excels from the short to long horizons predominantly, stressing the claim that it is more computationally efficient and effective again. Though in the main text, we have already compared BLUR to a couple Mamba type of models. To more comprehensively assess our model's performance when compared to this kind of state-of-the-art model, we present a case study in Table~\ref{table:timeseries_full_table6} by leveraging ETTm$_1$ dataset and setting a horizon equal to 96. Again, BLUR performs better than all Mamba types of models. Combining Table~\ref{table:timeseries_full_table4}, it implies that BLUR is better at timeseries prediction when the horizon is short to medium. This intuitively makes sense as Mamba has many more model parameters that can assist in the model expressivity. 

\begin{table*}[ht]
\centering
\caption{Multivariate long sequence time-series forecasting results} 
\resizebox{\textwidth}{!}{%
\begin{tabular}{|c|c|cc|cc|cc|cc|cc|cc|cc|cc|}
\hline
\multirow{2}{*}{Methods} & \multirow{2}{*}{Metric} & \multicolumn{2}{c|}{\textbf{BLUR}} & \multicolumn{2}{c|}{LRU} & \multicolumn{2}{c|}{S4} & \multicolumn{2}{c|}{Informer} & \multicolumn{2}{c|}{Informer†} & \multicolumn{2}{c|}{LogTrans} & \multicolumn{2}{c|}{Reformer} \\
 &  & MSE & MAE & MSE & MAE & MSE & MAE & MSE & MAE & MSE & MAE & MSE & MAE & MSE & MAE \\
\hline
\multirow{5}{*}{\rotatebox{90}{ETTh$_1$}} 
 & 24  & \textbf{0.151} &  \textbf{0.300}  & 0.210  & 0.319 & 0.525  & 0.542  & 0.577  & 0.549  & 0.620  & 0.577 & 0.686  & 0.604 & 0.991 & 0.754 \\
 & 48  & \textbf{0.271}  & 0.484  & 0.319 & \textbf{0.383} & 0.641  & 0.615  & 0.685  & 0.625  & 0.692  & 0.671 & 0.766  & 0.757 & 1.313 & 0.906\\
 & 168 & \textbf{0.397} & \textbf{0.508}  & 0.552 & 0.523 & 0.980  & 0.779  & 0.931  & 0.752  & 0.947  & 0.797 & 1.002  & 0.846 & 1.824 & 1.138 \\
 & 336 & \textbf{0.247} & \textbf{0.336}  & 0.813 & 0.696  & 1.407  & 0.910  & 1.128  & 0.873  & 1.094  & 0.813 & 1.362  & 0.952 & 2.117 & 1.280\\
 & 720 & \textbf{0.450} & \textbf{0.536}  & 1.214 & 0.880 & 1.162  & 0.842  & 1.215  & 0.896  & 1.241  & 0.917 & 1.397  & 1.291 & 2.415 & 1.520 \\
\hline
\multirow{5}{*}{\rotatebox{90}{ETTh$_2$}} 
& 24 & \textbf{0.019} & \textbf{0.108} & 0.107 & 0.230 & 0.871 & 0.736 & 0.720 & 0.665 & 0.753 & 0.727 & 0.828 & 0.750 & 1.531 & 1.613 \\
& 48 & \textbf{0.073} & \textbf{0.214} & 0.156 & 0.539 & 1.240 & 0.867 & 1.457 & 1.001 & 1.461 & 1.077 & 1.806 & 1.034 & 1.871 & 1.735 \\
& 168 & \textbf{0.965} & \textbf{0.803} & 2.080 & 1.030 & 2.580 & 1.255 & 3.489 & 1.515 & 3.485 & 1.612 & 4.070 & 1.681 & 4.660 & 1.846 \\
& 336 & \textbf{1.030} & \textbf{0.853} & 2.841 & 1.160 & 1.980 & 1.128 & 2.723 & 1.340 & 2.626 & 1.285 & 3.875 & 1.763 & 4.028 & 1.688 \\
& 720 & \textbf{0.594} & \textbf{0.640} & 2.949 & 2.050 & 2.650 & 1.340 & 3.467 & 1.473 & 3.548 & 1.495 & 3.913 & 1.552 & 5.381 & 2.015 \\ 
\hline
\multirow{5}{*}{\rotatebox{90}{ETTm$_1$}} 
 & 24  & \textbf{0.029} & \textbf{0.121}  & 0.093 & 0.209  & 0.426  & 0.487  & 0.323  & 0.369  & 0.306  & 0.371 & 0.419  & 0.412 & 0.724 & 0.607\\
 & 48  &  \textbf{0.057} & \textbf{0.220}  & 0.167 &  0.288 & 0.580  & 0.565  & 0.494  & 0.503  & 0.465  & 0.470 & 0.507  & 0.583 & 1.098 & 0.777\\
 & 96  &  \textbf{0.150} & \textbf{0.337}  & 0.236 & 0.339 & 0.699  & 0.649  & 0.678  & 0.614  & 0.681  & 0.612 & 0.768  & 0.792 & 1.433 & 0.945\\
 & 288 & \textbf{0.375} &0.550  & 0.420  & \textbf{0.465} & 0.824  & 0.674  & 1.056  & 0.786  & 1.162  & 0.879 & 1.462  & 1.320 & 1.820 & 1.094\\
 & 672 & \textbf{0.545}  & \textbf{0.644}  & 0.807 & 0.672  & 0.846  & 0.709  & 1.192  & 0.926  & 1.231  & 1.103 & 1.669  & 1.461 & 2.187 & 1.232\\
\hline
\multirow{5}{*}{\rotatebox{90}{Electricity}} 
 & 48  & \textbf{0.221} & \textbf{0.340} & 0.245  & 0.360 & 0.255  & 0.352  & 0.344  & 0.393  & 0.334  & 0.399 & 0.355  & 0.418 & 1.404 & 0.999\\
 & 168 & \textbf{0.244}  & \textbf{0.357}  & 0.261 & 0.383  & 0.283  & 0.373  & 0.368  & 0.424  & 0.353  & 0.420 & 0.368  & 0.432 & 1.515 & 1.069\\
 & 336 & \textbf{0.274} & \textbf{0.387}  & 0.283 & 0.395  & 0.292  & 0.382  & 0.381  & 0.431  & 0.381  & 0.439 & 0.373  & 0.439 & 1.601 & 1.104\\
 & 720 & 0.418  & 0.505  & 0.398 &  0.478 & \textbf{0.289} & \textbf{0.377}  & 0.406  & 0.443  & 0.391  & 0.438 & 0.409  & 0.454 & 2.009 & 1.170\\
 & 960 & 0.773 & 0.704  & 0.465 & 0.510 & \textbf{0.299}  & \textbf{0.387}  & 0.460  & 0.548  & 0.492  & 0.550 & 0.477  & 0.589 &2.141 &1.387 \\
\hline
\multirow{5}{*}{\rotatebox{90}{Weather}}  
 & 24  & \textbf{0.075} & \textbf{0.186}  & 0.225 & 0.238 & 0.334  & 0.385  & 0.335  & 0.381  & 0.349  & 0.397 & 0.435  & 0.477 & 0.655 & 0.583\\
 & 48  & \textbf{0.104}  & \textbf{0.232}  & 0.265 & 0.290 & 0.406  & 0.444  & 0.395  & 0.459  & 0.386  & 0.433 & 0.426  & 0.495 & 0.729 & 0.666\\
 & 168 & \textbf{0.197} &  \textbf{0.332}  & 0.395 & 0.413 & 0.525  & 0.527  & 0.608  & 0.567  & 0.613  & 0.582 & 0.727  & 0.671 & 1.318 & 0.855\\
 & 336 & \textbf{0.251} &  \textbf{0.362}  & 0.486 & 0.472 & 0.531  & 0.539  & 0.702  & 0.620  & 0.707  & 0.634 & 0.754  & 0.670 & 1.930 & 1.167\\
 & 720 & \textbf{0.325} &  \textbf{0.441}  & 0.547 & 0.528 & 0.578  & 0.578  & 0.831  & 0.731  & 0.834  & 0.741 & 0.885  & 0.773 & 2.726 & 1.575\\
\hline

Count& - & \multicolumn{2}{c|}{22} & \multicolumn{2}{c|}{1} & \multicolumn{2}{c|}{2} & \multicolumn{2}{c|}{0} & \multicolumn{2}{c|}{0} & \multicolumn{2}{c|}{0} & \multicolumn{2}{c|}{0}\\ 
\hline
\end{tabular}%
}
\label{table:timeseries_full_table1}
\end{table*}

\begin{table*}[ht]
\centering
\caption{Multivariate long sequence time-series forecasting results}
\resizebox{\textwidth}{!}{%
\begin{tabular}{|c|c|cc|cc|cc|cc|cc|cc|cc|cc|}
\hline
\multirow{2}{*}{Methods} & \multirow{2}{*}{Metric} & \multicolumn{2}{c|}{\textbf{BLUR}} & \multicolumn{2}{c|}{CoST} & \multicolumn{2}{c|}{TS2Vec} & \multicolumn{2}{c|}{TNC} & \multicolumn{2}{c|}{MoCo} & \multicolumn{2}{c|}{TCN} & \multicolumn{2}{c|}{LSTnet} \\
 &  & MSE & MAE & MSE & MAE & MSE & MAE & MSE & MAE & MSE & MAE & MSE & MAE & MSE & MAE \\
\hline
\multirow{5}{*}{\rotatebox{90}{ETTh$_1$}} 
 & 24  & \textbf{0.151} & \textbf{0.300}  & 0.386 & 0.429 & 0.590 & 0.531 & 0.708 & 0.592 & 0.623 & 0.555 & 0.583 & 0.547 & 1.293 & 0.901 \\
 & 48  & \textbf{0.271}  & 0.484  & 0.437 & \textbf{0.464} & 0.624 & 0.555 & 0.749 & 0.619 & 0.669 & 0.586 & 0.670 & 0.606 & 1.456 & 0.960 \\
 & 168 & \textbf{0.397} & \textbf{0.508}  & 0.643 & 0.582 & 0.762 & 0.639 & 0.884 & 0.690 & 0.820 & 0.674 & 0.811 & 0.680 & 1.997 & 1.214 \\
 & 336 & \textbf{0.247} & \textbf{0.336}  & 0.812 & 0.679 & 0.931 & 0.728 & 1.020 & 0.768 & 0.981 & 0.755 & 1.132 & 0.815 & 2.655 & 1.369 \\
 & 720 & \textbf{0.450} & \textbf{0.536}  & 0.970 & 0.771 & 1.063 & 0.799 & 1.157 & 0.838 & 1.138 & 0.831 & 1.165 & 0.813 & 2.143 & 1.380 \\
\hline
\multirow{5}{*}{\rotatebox{90}{ETTh$_2$}} 
 & 24  & \textbf{0.019} & \textbf{0.108} & 0.447 & 0.502 & 0.423 & 0.489 & 0.616 & 0.595 & 0.444 & 0.495 & 0.935 & 0.754 & 2.742 & 1.457 \\
 & 48  & \textbf{0.073} & \textbf{0.214}  & 0.699 & 0.637 & 0.619 & 0.605 & 0.840 & 0.716 & 0.613 & 0.595 & 1.300 & 0.911 & 3.567 & 1.687 \\
 & 168 & \textbf{0.965} & \textbf{0.803} & 1.549 & 0.982 & 1.845 & 1.074 & 2.359 & 1.213 & 1.791 & 1.034 & 4.017 & 1.579 & 3.242 & 2.513 \\
 & 336 & \textbf{1.030} & \textbf{0.853}  & 1.749 & 1.042 & 2.194 & 1.197 & 2.789 & 1.342 & 2.241 & 1.186 & 3.460 & 1.456 & 2.544 & 2.591 \\
 & 720 & \textbf{0.594} & \textbf{0.640} & 1.971 & 1.092 & 2.636 & 1.370 & 2.753 & 1.394 & 2.425 & 1.292 & 3.106 & 1.381 & 4.625 & 3.709 \\
\hline
\multirow{5}{*}{\rotatebox{90}{ETTm$_1$}} 
 & 24  & \textbf{0.029} & \textbf{0.121}  & 0.246 & 0.329 & 0.453 & 0.444 & 0.522 & 0.472 & 0.458 & 0.444 & 0.363 & 0.397 & 1.968 & 1.170 \\
 & 48  &  \textbf{0.057} & \textbf{0.220}  & 0.331 & 0.386 & 0.592 & 0.521 & 0.695 & 0.567 & 0.594 & 0.528 & 0.542 & 0.508 & 1.999 & 1.215 \\
 & 96  &  \textbf{0.150} & \textbf{0.337}  & 0.378 & 0.419 & 0.653 & 0.554 & 0.913 & 0.674 & 0.816 & 0.768 & 0.585 & 0.578 & 2.762 & 1.542 \\
 & 288 & \textbf{0.375} & 0.550  & 0.472 & \textbf{0.486} & 0.693 & 0.597 & 0.831 & 0.648 & 0.876 & 0.755 & 0.735 & 0.735 & 1.257 & 2.076 \\
 & 672 & \textbf{0.545}  & 0.644 & 0.620 & \textbf{0.574} & 0.782 & 0.653 & 0.932 & 0.826 & 0.821 & 0.674 & 1.032 & 0.756 & 1.917 & 2.941 \\
\hline
\multirow{4}{*}{\rotatebox{90}{Electricity}}  
 & 48  & 0.221  & 0.340  & \textbf{0.153} & \textbf{0.258} & 0.309 & 0.391 & 0.375 & 0.431 & 0.316 & 0.389 & 0.253 & 0.359 & 0.429 & 0.456 \\
 & 168 & \textbf{0.244} &  \textbf{0.357}  & 0.515 & 0.509 & 0.506 & 0.518 & 0.358 & 0.423 & 0.426 & 0.466 & 0.372 & 0.425 & 0.372 & 0.425 \\
 & 336 & 0.274 &  0.387  & \textbf{0.196} & \textbf{0.296} & 0.351 & 0.422 & 0.417 & 0.465 & 0.325 & 0.417 & 0.352 & 0.409 & 0.352 & 0.409 \\
 & 720 & 0.418 &  0.505 & \textbf{0.276} & \textbf{0.342} & 0.368 & 0.475 & 0.469 & 0.519 & 0.416 & 0.508 & 0.287 & 0.381 & 0.380 & 0.443 \\
\hline
\multirow{5}{*}{\rotatebox{90}{Weather}}  
 & 24  & \textbf{0.075} & \textbf{0.186}   & 0.298 & 0.360 & 0.307 & 0.363 & 0.320 & 0.373 & 0.311 & 0.365 & 0.321 & 0.367 & 0.615 & 0.545 \\
& 48  & \textbf{0.104}  & \textbf{0.232} & 0.359 & 0.411 & 0.374 & 0.418 & 0.358 & 0.435 & 0.398 & 0.456 & 0.386 & 0.423& 0.660 & 0.589 \\
 & 168 & \textbf{0.197} &  \textbf{0.332} &0.464 & 0.491 & 0.491 & 0.506 & 0.479 & 0.495 & 0.482 & 0.499& 0.491 & 0.501& 0.748 & 0.647\\
& 336 & \textbf{0.251} &  \textbf{0.362} & 0.497 & 0.517 & 0.525 & 0.530 & 0.505 & 0.514 & 0.516 & 0.523 & 0.502 & 0.507& 0.782 & 0.683\\
 & 720 & \textbf{0.325} &  \textbf{0.441}  & 0.533 & 0.542 & 0.582 & 0.540 & 0.611 & 0.532 & 0.508 & 0.507 & 0.498 & 0.508& 0.854 & 0.757 \\
\hline
Count& - & \multicolumn{2}{c|}{19} & \multicolumn{2}{c|}{5} & \multicolumn{2}{c|}{0} & \multicolumn{2}{c|}{0} & \multicolumn{2}{c|}{0} & \multicolumn{2}{c|}{0} & \multicolumn{2}{c|}{0}\\ 
\hline
\end{tabular}%
}
\label{table:timeseries_full_table2}
\end{table*}

\begin{table*}[ht]
\centering
\caption{Multivariate long sequence time-series forecasting results}
\resizebox{0.6\textwidth}{!}{%
\begin{tabular}{|c|c|cc|cc|cc|cc|}
\hline
\multirow{2}{*}{Methods} & \multirow{2}{*}{Metric} & \multicolumn{2}{c|}{\textbf{BLUR}} & \multicolumn{2}{c|}{LSTMa} & \multicolumn{2}{c|}{Autoformer} & \multicolumn{2}{c|}{Pyraformer} \\
 &  & MSE & MAE & MSE & MAE & MSE & MAE & MSE & MAE \\
\hline
\multirow{5}{*}{\rotatebox{90}{ETTh$_1$}} 
 & 24  & \textbf{0.151} & \textbf{0.300}  & 0.650 & 0.624 & 0.384 & 0.425 & 0.479 & 0.499 \\
 & 48  & \textbf{0.271} & 0.484  & 0.702 & 0.675 & 0.392 & \textbf{0.419} & 0.518 & 0.520 \\
 & 168 & \textbf{0.397} & 0.508  & 1.212 & 0.867 & 0.490 & \textbf{0.481} & 0.758 & 0.665 \\
 & 336 & \textbf{0.247} & \textbf{0.336}  & 1.424 & 0.994 & 0.505 & 0.484 & 0.891 & 0.738 \\
 & 720 & \textbf{0.450} & 0.536  & 1.960 & 1.322 & 0.498 & \textbf{0.500} & 0.963 & 0.782 \\
\hline
\multirow{5}{*}{\rotatebox{90}{ETTh$_2$}} 
 & 24  & \textbf{0.019} & \textbf{0.108}  & 1.143 & 0.813 & 0.261 & 0.341 & 0.477 & 0.537 \\
 & 48  & \textbf{0.073} & \textbf{0.214}  & 1.671 & 1.221 & 0.312 & 0.373 & 0.934 & 0.764 \\
 & 168 & 0.965 & 0.803 & 4.117 & 1.674 & \textbf{0.457} & \textbf{0.455} & 3.913 & 1.557 \\
 & 336 & 1.030 & 0.853 & 3.434 & 1.549 & \textbf{0.471} & \textbf{0.475} & 0.907 & 0.747 \\
 & 720 & 0.594 & 0.640  & 3.963 & 1.788 & \textbf{0.474} & \textbf{0.484} & 0.963 & 0.783 \\
\hline
\multirow{5}{*}{\rotatebox{90}{ETTm$_1$}} 
 & 24  & \textbf{0.029} & \textbf{0.121}  & 0.621 & 0.629 & 0.383 & 0.403 & 0.332 & 0.383 \\
 & 48  & \textbf{0.057} & \textbf{0.220}  & 1.392 & 0.939 & 0.454 & 0.453 & 0.492 & 0.475 \\
 & 96  & \textbf{0.150} & \textbf{0.337}  & 1.339 & 0.913 & 0.481 & 0.463 & 0.543 & 0.510 \\
 & 288 & \textbf{0.375} & 0.550  & 1.740 & 1.124 & 0.634 & \textbf{0.528} & 0.656 & 0.598 \\
 & 672 & \textbf{0.545} & 0.644  & 2.736 & 1.555 & 0.606 & \textbf{0.542} & 0.901 & 0.720 \\
\hline
\multirow{5}{*}{\rotatebox{90}{Weather}}  
 & 24  & \textbf{0.075} & \textbf{0.186}  & 0.546 & 0.570 & -- & -- & -- & -- \\
 & 48  & \textbf{0.104} & \textbf{0.232}  & 0.829 & 0.677 & -- & -- & -- & -- \\
 & 168 & \textbf{0.197} & \textbf{0.332}  & 1.038 & 0.835 & -- & -- & -- & -- \\
 & 336 & \textbf{0.251} & \textbf{0.362}  & 1.657 & 1.059 & 0.359 & 0.395 & -- & -- \\
 & 720 & \textbf{0.325} & 0.441  & 1.536 & 1.109 & 0.419 &\textbf{0.428} & -- & -- \\
\hline
Count& - & \multicolumn{2}{c|}{14} & \multicolumn{2}{c|}{0} & \multicolumn{2}{c|}{6} & \multicolumn{2}{c|}{0} \\ 
\hline
\end{tabular}%
}
\label{table:timeseries_full_table3}
\end{table*}

\begin{table*}[ht]
\centering
\caption{Multivariate long sequence time-series forecasting results}
\resizebox{\textwidth}{!}{%
\begin{tabular}{|c|c|cc|cc|cc|cc|cc|cc|cc|}
\hline
\multirow{2}{*}{Methods} & \multirow{2}{*}{Metric} & \multicolumn{2}{c|}{\textbf{BLUR}} & \multicolumn{2}{c|}{S-Mamba} & \multicolumn{2}{c|}{iTransformer} & \multicolumn{2}{c|}{RLinear} & \multicolumn{2}{c|}{PatchTST} & \multicolumn{2}{c|}{Crossformer} & \multicolumn{2}{c|}{TiDE} \\
 &  & MSE & MAE & MSE & MAE & MSE & MAE & MSE & MAE & MSE & MAE & MSE & MAE & MSE & MAE \\
\hline
\multirow{2}{*}{{ETTh$_1$}} 
 & 336 & \textbf{0.247} & \textbf{0.336}  & 0.408 & 0.413 & 0.420 & 0.415 & 0.424 & 0.415 & 0.399 & 0.410 & 0.532 & 0.515 & 0.428 & 0.454 \\
 & 720 & \textbf{0.450} & 0.536  & 0.475 & 0.448 & 0.491 & 0.459 & 0.487 & 0.450 & 0.454 & \textbf{0.439} & 0.666 & 0.589 & 0.487 & 0.461 \\
\hline
\multirow{2}{*}{{ETTh$_2$}} 
 & 336 & 1.030 & 0.853 & \textbf{0.424} & 0.431 & 0.426 & \textbf{0.426} & 0.426 & 0.433 & 1.043 & 0.731 & 0.673 & 0.613 & 0.505 & 0.671 \\
 & 720 & 0.594 & 0.640 & 0.620 & \textbf{0.445} & 0.678 & 0.449 & \textbf{0.431} & 0.446 & 1.104 & 1.123 & 0.874 & 0.679 & 0.671 & 0.694 \\
\hline
\multirow{1}{*}{{ETTm$_1$}} 
 & 96  & \textbf{0.150} & \textbf{0.337}  & 0.333 & 0.368 & 0.334 & 0.368 & 0.355 & 0.376 & 0.329 & 0.367 & 0.404 & 0.426 & 0.364 & 0.387 \\
\hline
\multirow{2}{*}{{Weather}}  
 & 336 & \textbf{0.251} & 0.362  & 0.274 & \textbf{0.297} & 0.278 & \textbf{0.297} &0.300 & 0.317 & 0.272 & \textbf{0.297} & 0.272 & 0.335 & 0.287 & 0.303 \\
 & 720 & \textbf{0.325} & 0.441  & 0.350 & \textbf{0.345} & 0.358 & 0.347 & 0.374 & 0.365& 0.354 & 0.348 & 0.398 & 0.418 & 0.351 & 0.386 \\
\hline
Count& - & \multicolumn{2}{c|}{3} & \multicolumn{2}{c|}{2} & \multicolumn{2}{c|}{1} & \multicolumn{2}{c|}{0} & \multicolumn{2}{c|}{1} & \multicolumn{2}{c|}{0} & \multicolumn{2}{c|}{0}\\ 
\hline
\end{tabular}%
}
\label{table:timeseries_full_table4}
\end{table*}

\begin{table}[]
\centering
\caption{Multivariate long sequence time-series forecasting results on four datasets (five cases) with the comparison to BRNN and BLSTM.} 
\begin{tabular}{|c|c|cc|cc|cc|}
\hline
\multirow{2}{*}{Methods} &
  \multirow{2}{*}{Metric} &
  \multicolumn{2}{c|}{\textbf{BLUR}} &
  \multicolumn{2}{c|}{BRNN} &
  \multicolumn{2}{c|}{BLSTM} \\ 
 &
   &
  MSE &
  MAE &
  MSE &
  MAE &
  MSE &
  MAE \\ \hline
\multirow{5}{*}{ETTh$_1$} 
& 24  & \textbf{0.151}  & \textbf{0.300} & 0.578 & 0.496 & 0.622 & 0.518          \\
& 48  & \textbf{0.271}  & \textbf{0.484} & 0.678 & 0.557 & 0.650 & 0.576          \\
& 168 & \textbf{0.397}  & \textbf{0.508} & 0.976 & 0.754 & 0.968 & 0.729          \\
& 336 & \textbf{0.247}  & \textbf{0.336} & 1.243 & 0.829 & 1.150  & 0.787          \\
& 720 & \textbf{0.450}  & \textbf{0.536} & 1.396 & 0.938 & 1.423 & 0.905          \\ \hline
\multirow{5}{*}{ETTh$_2$} 
& 24  & \textbf{0.019} & \textbf{0.108}         & 0.298 & 0.332 & 0.306 & 0.365          \\
& 48  & \textbf{0.073} & \textbf{0.214}        & 0.538 & 0.570 & 0.589 & 0.478          \\
& 168 & \textbf{0.965} & \textbf{0.803}         & 2.630  & 1.050  & 3.150  & 1.002          \\
& 336 & \textbf{1.030} & \textbf{0.853}         & 2.750  & 1.116 & 2.350  & 0.976 \\
& 720 & \textbf{0.594} & \textbf{0.640}        & 2.830  & 1.130  & 2.798 & 0.990  \\ \hline
\multirow{5}{*}{ETTm$_1$} 
& 24  & \textbf{0.029} & \textbf{0.121} & 0.456 & 0.450 & 0.489 & 0.462          \\
& 48  & \textbf{0.057} & \textbf{0.220} & 0.721 & 0.593 & 0.692 & 0.573          \\
& 96  & \textbf{0.150} & \textbf{0.337} & 0.678 & 0.576 & 0.625 & 0.556          \\
& 288 & \textbf{0.375}  & \textbf{0.550} & 0.883 & 0.701 & 0.821 & 0.667          \\
& 672 & \textbf{0.545}  & \textbf{0.644} & 1.085 & 0.778 & 0.966 & 0.752          \\ \hline
Count                                            &   -  & \multicolumn{2}{c|}{15}             & \multicolumn{2}{c|}{0}     & \multicolumn{2}{c|}{0}             \\ \hline
\end{tabular}
\label{table:timeseries_full_table5}
\end{table}

\begin{table}[htp]
\caption{Quantitative comparison of approaches between \textit{Mamba} type of models and BLUR for ETTm$_1$ dataset with horizon = 96.}
\begin{center} 
\begin{tabular}{c c c}
    \toprule
    \textbf{Method} & \textbf{MSE} & \textbf{MAE}\\ \midrule
    SiMBA~\cite{patro2024simba} & 0.324 & 0.360 \\
    S-Mamba~\cite{wang2025mamba}&0.333& 0.368\\
    CMMamba~\cite{li2024cmmamba} & 0.293 & 0.347 \\
    Bi-Mamba+\cite{liang2024bi} & 0.320 & 0.360\\
    UMambaTSF~\cite{wu2024umambatsf} & 0.316 & 0.356\\
      TimeMachine~\cite{ahamed2024timemachine}   &              0.317 &0.355     \\
      SST~\cite{xu2024integrating} &0.298&0.355\\
      \textbf{BLUR}   &     \textbf{0.150} & \textbf{0.337}               \\ 
      \bottomrule
      
\end{tabular}
\end{center}
\label{table:timeseries_full_table6}
\end{table}


We also validated our findings using several additional datasets. In Figures 
~\ref{fig:groundtruth_prediction_ETTH1},
~\ref{fig:groundtruth_prediction_full_ETTH2},  ~\ref{fig:groundtruth_prediction_full_ETTM1}, ~\ref{fig:groundtruth_prediction_full_WTH}, and ~\ref{fig:groundtruth_prediction_full_ECL}, we present the comparison for ETTh$_1$, ETTh$_2$, ETTm$_1$, and the Weather dataset with a horizon of 24, and for Electricity with a horizon of 48. Although both LRU and BLUR effectively capture the overall trends of the ground truth, BLUR still outperforms LRU with a closer fit to the ground truth data.

\begin{figure}
    \centering
    \includegraphics[width=0.8\linewidth]{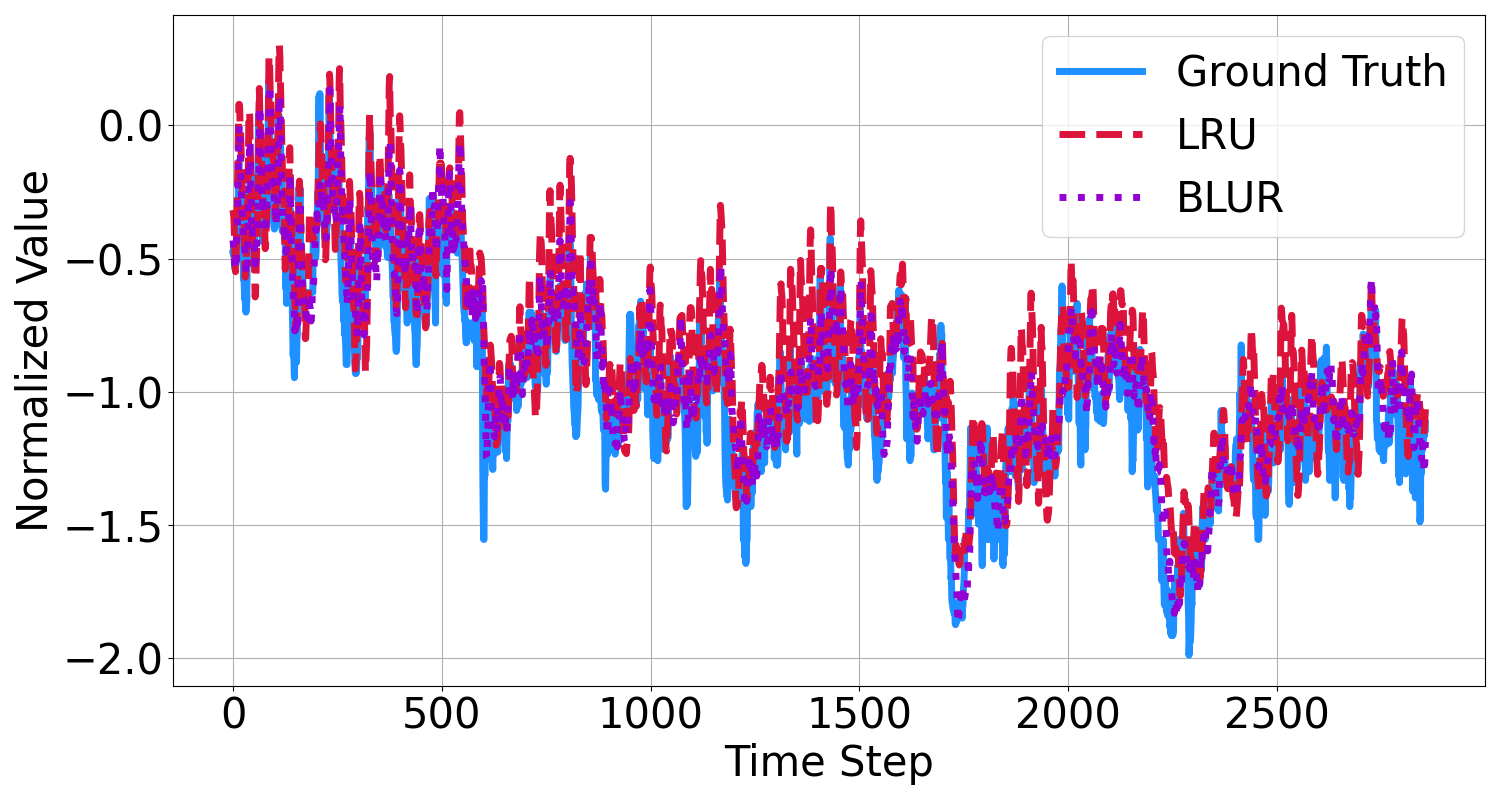}
    \caption{Model comparison between LRU and BLUR on the full testing data of ETTh$_1$ for the prediction.}
    \label{fig:groundtruth_prediction_ETTH1}
\end{figure}

\begin{figure}[h!]
    \centering
    \includegraphics[width=0.8\linewidth]{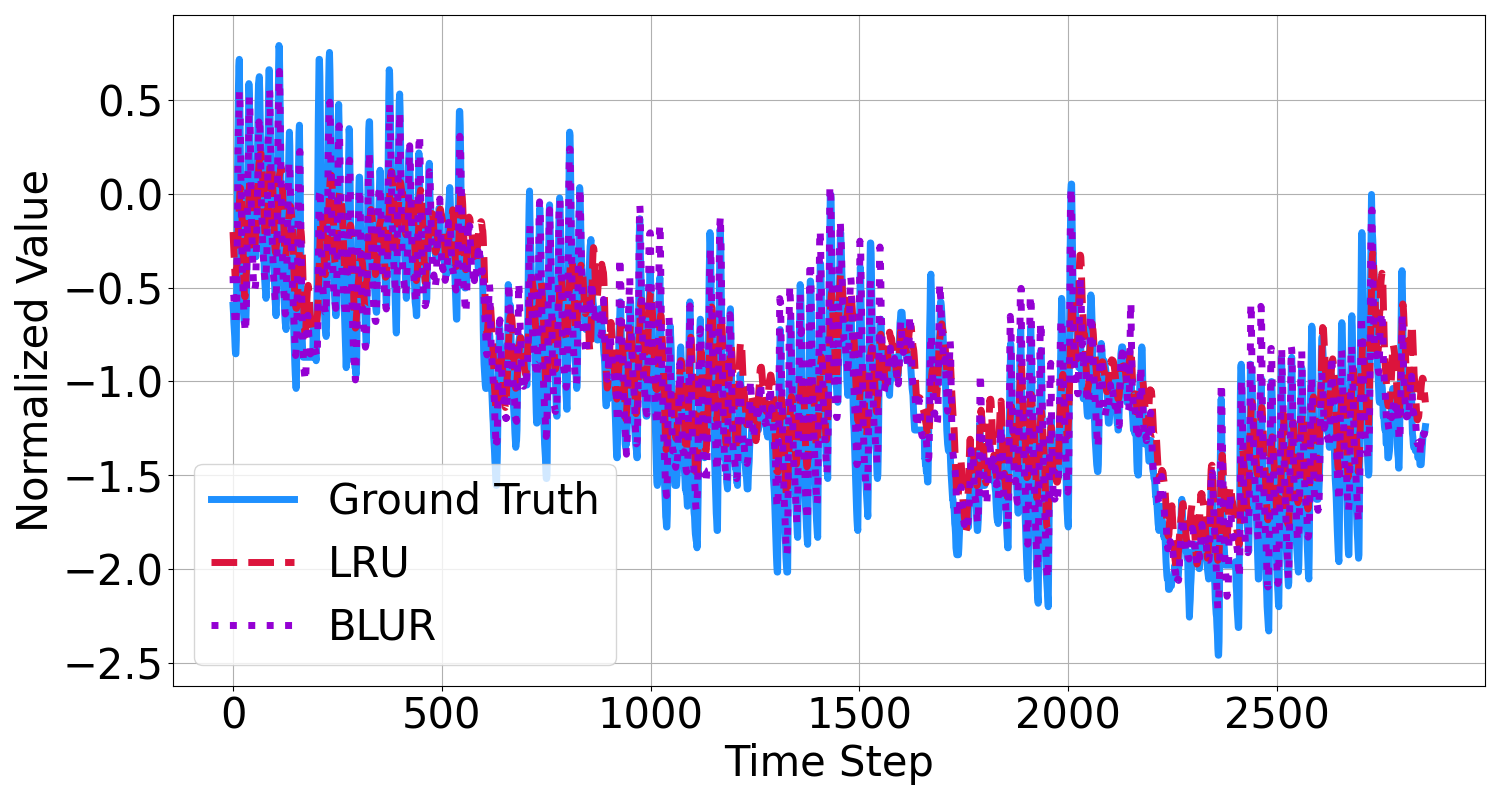}
    \caption{Model comparison between LRU and BLUR on the full testing data of ETTh$_2$ for the prediction.}
    \label{fig:groundtruth_prediction_full_ETTH2}
\end{figure}

\begin{figure}[h]
    \centering
    \includegraphics[width=0.8\linewidth]{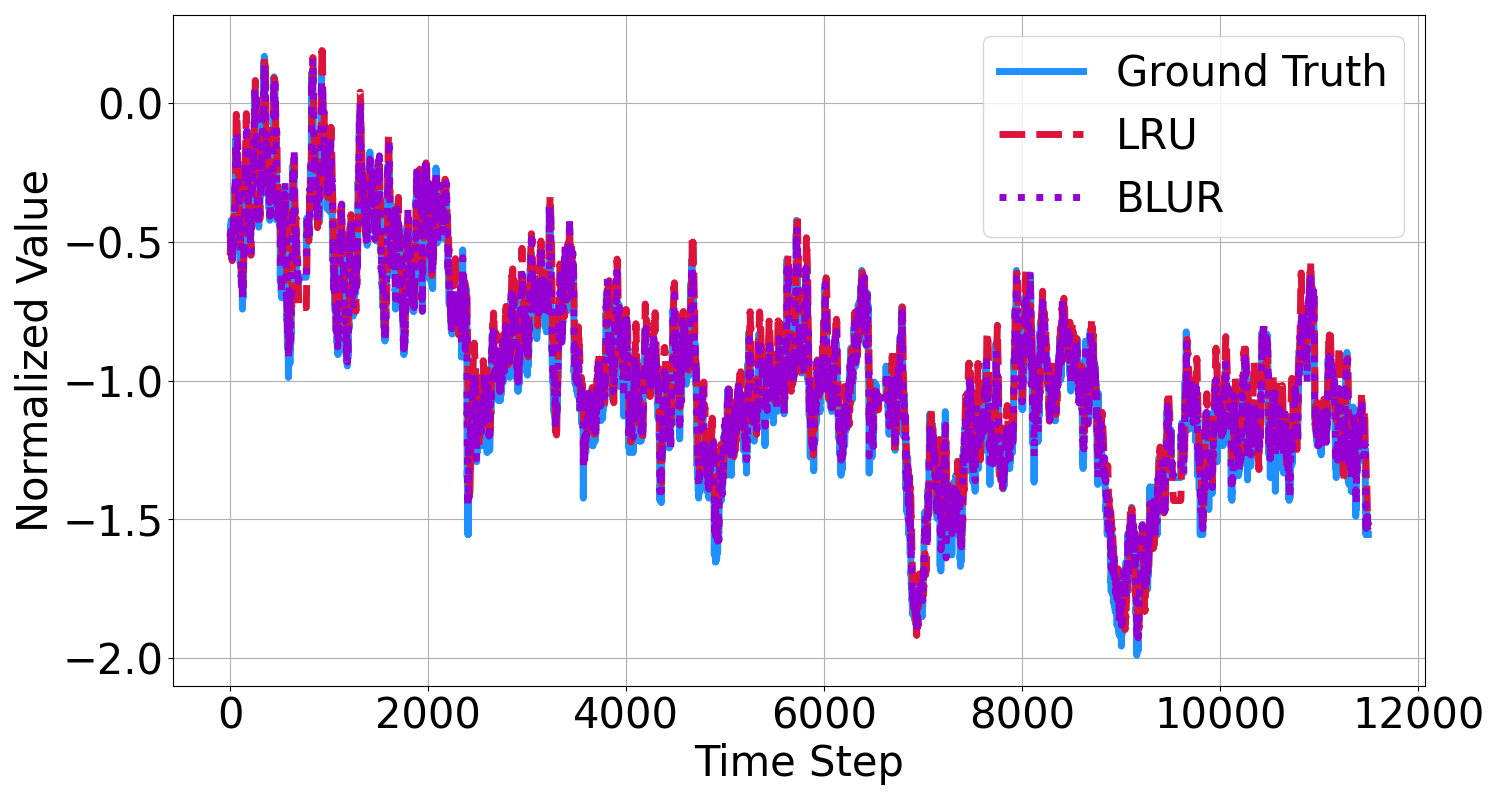}
    \caption{Model comparison between LRU and BLUR on the full testing data of ETTm$_1$ for the prediction.}
    \label{fig:groundtruth_prediction_full_ETTM1}
\end{figure}

\begin{figure}[h]
    \centering
    \includegraphics[width=0.8\linewidth]{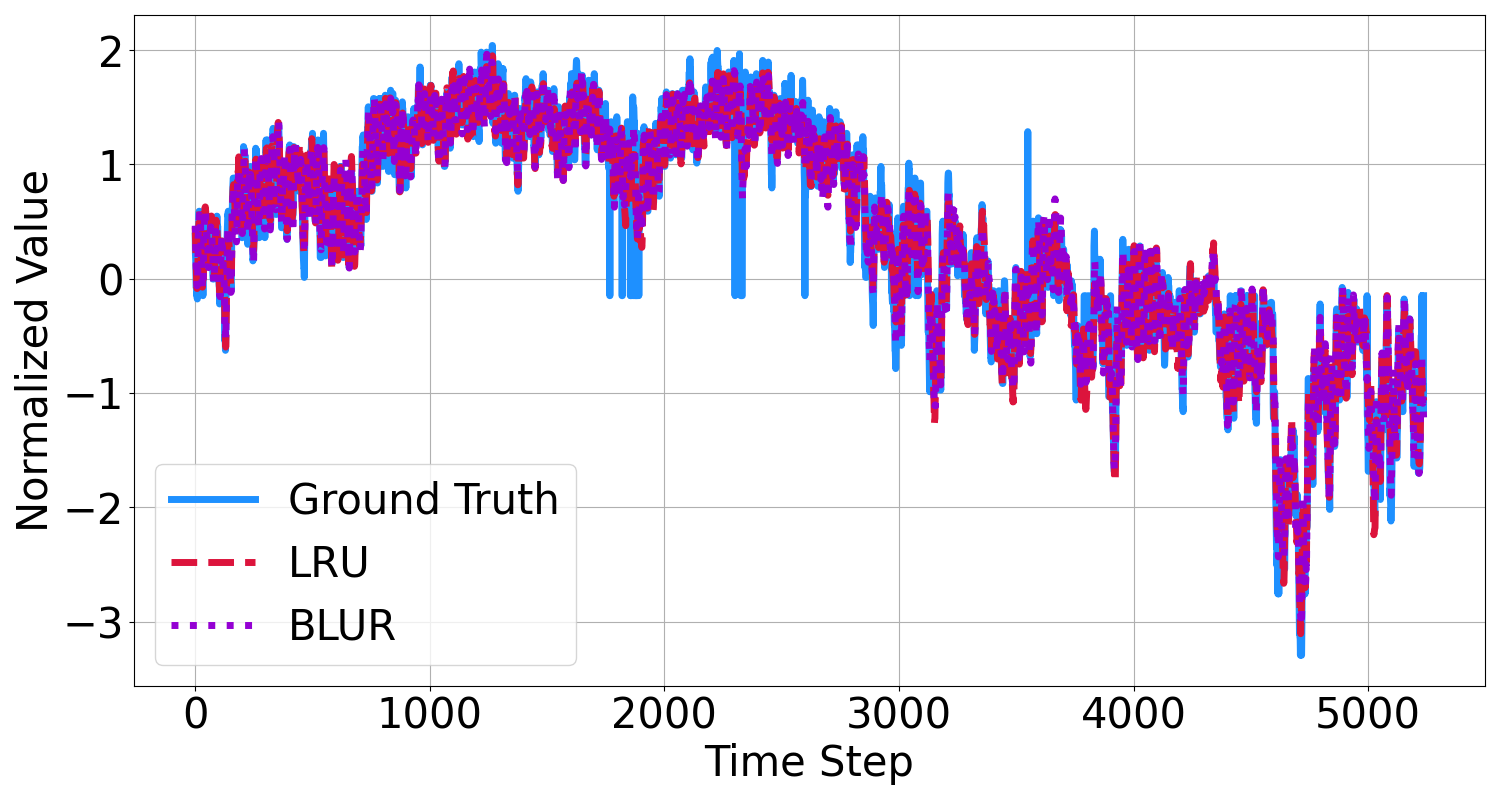}
    \caption{Model comparison between LRU and BLUR on the full testing data of Weather for the prediction.}
    \label{fig:groundtruth_prediction_full_WTH}
\end{figure}

\begin{figure}[h]
    \centering
    \includegraphics[width=0.8\linewidth]{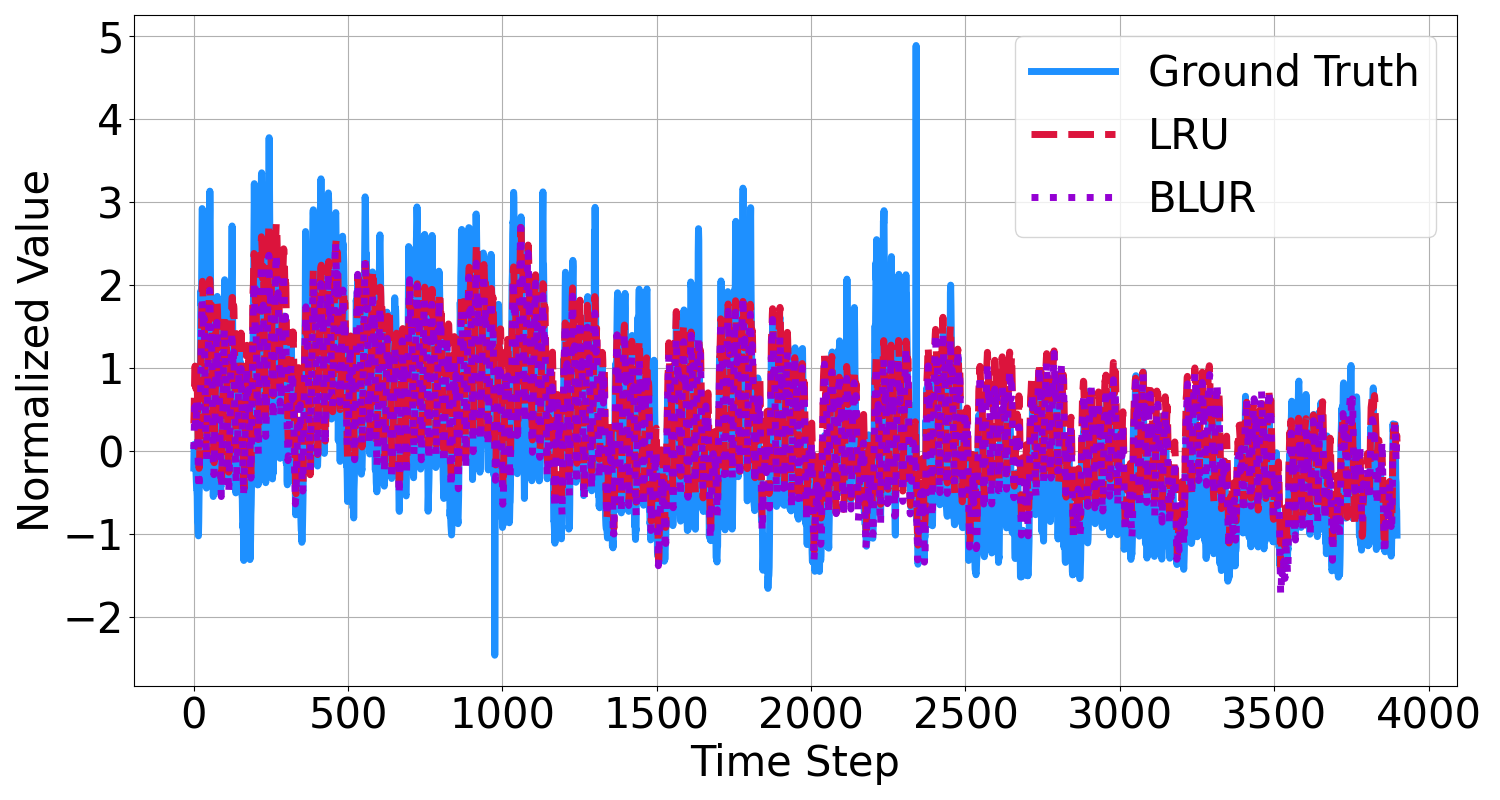}
    \caption{Model comparison between LRU and BLUR on the full testing data of Electricity for the prediction.}
    \label{fig:groundtruth_prediction_full_ECL}
\end{figure}
\textbf{Discussion on Model Performance. } Figure~\ref{fig:groundtruth_prediction_ETTH1} presents the prediction results on the ETTh$_1$ dataset with a prediction horizon of 24. The red dashed line represents the baseline LRU model, while the purple dotted line corresponds to our proposed BLUR model. The blue solid line denotes the ground truth. Overall, both models successfully capture the general trend of the ground truth. However, BLUR demonstrates a noticeably closer fit compared to LRU, particularly in capturing sharp fluctuations. For instance, at around time step 1800, BLUR accurately tracks the spike, whereas LRU exhibits a delayed and less pronounced response. 

Figure~\ref{fig:groundtruth_prediction_full_ETTH2} illustrates the results for the ETTh$_2$ dataset. Similar to ETTh$_1$, both models align closely with the ground truth, but BLUR demonstrates a more refined representation in sections with high volatility. A notable example is between time steps 1000 and 1500, where BLUR maintains tighter alignment with the peaks and troughs, while LRU tends to oversmooth the trajectory.

Figure~\ref{fig:groundtruth_prediction_full_ETTM1} showcases the results on ETTm$_1$ dataset. In this scenario, BLUR’s adaptability becomes more apparent. For instance, around time step 11000, the ground truth undergoes a series of rapid oscillations. BLUR successfully tracks these oscillations with minimal lag, whereas LRU slightly leads ahead, leading to a phase shift in its predictions. 

Figure~\ref{fig:groundtruth_prediction_full_WTH} evaluates the models on the WTH dataset, where both models perform well in capturing the overall trend. However, BLUR exhibits a slightly more precise alignment, particularly in the latter part of the time series. For instance, at time step 5000, LRU clearly overestimates the peak, whereas BLUR maintains a relatively more accurate prediction.

Figure~\ref{fig:groundtruth_prediction_full_ECL} presents the results on the Electricity dataset with a prediction horizon of 48. Compared to the previous datasets, Electricity exhibits high-frequency variations with multiple peaks and troughs, making accurate forecasting more challenging. Both models effectively track the overall trend of the ground truth, but BLUR shows a slightly better alignment, especially in high-variance regions. For instance, between time steps 3000 and 3500, LRU exhibits a slight overestimation of peak values, whereas BLUR maintains a closer fit to the ground truth. Additionally, in regions with sharp downward movements, BLUR adjusts more smoothly, avoiding the overshooting behavior observed in LRU.

\textbf{Computational time.}
Since the proposed BLUR inherits the parallel training capability from LRU, which has been shown to significantly reduce the computational cost compared to traditional RNNs and transformers, in this context, we present the practical runtime results to scrutinize if BLUR would suffer high computational time issue. Figures~\ref{fig:time_comparison_classification} and~\ref{fig:time_comparison_regression} show the wall-clock time of running one epoch for different models in the same machine. The immediate observation from the plot is that LRU/BLUR benefits considerably from the linear recurrence block that enables parallel training. BLUR is $\sim$3x and $\sim$72x computationally faster than S4/S5 and Informer, respectively, in classification and prediction tasks. However, we would like to note that in the timeseries prediction task, such high computational efficiency is also partially attributed to the usage of the JAX platform, which has been justified in~\cite{smith2022simplified} for the S5 model. Additionally, based on the report from Hugging Face, neither Informer nor Autoformer is supported by Flax, which means that both of them are not implemented by using JAX as of now. This imposes difficulty on the direct computational time comparison between BLUR and either of them.
Regardless of computational platforms, the prediction performance motivates us to resort to linear RNNs in modeling long-term dependencies.

\begin{figure}
    \centering
    \includegraphics[width=0.65\linewidth]{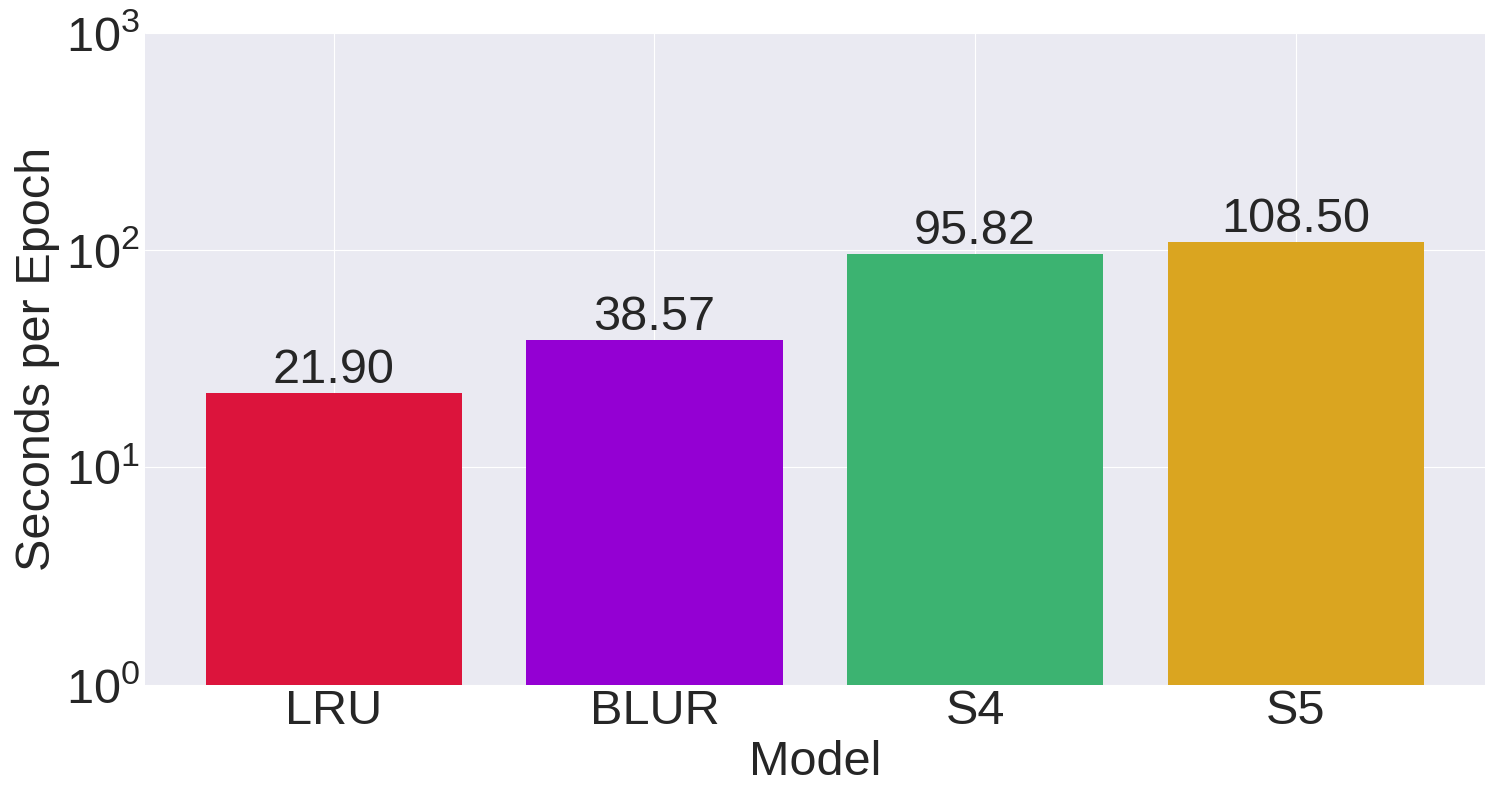}
    \caption{Training time costs per epoch for different methods with Sequential CIFAR.}
    \label{fig:time_comparison_classification}
\end{figure}

\begin{figure}
    \centering
    \includegraphics[width=0.65\linewidth]{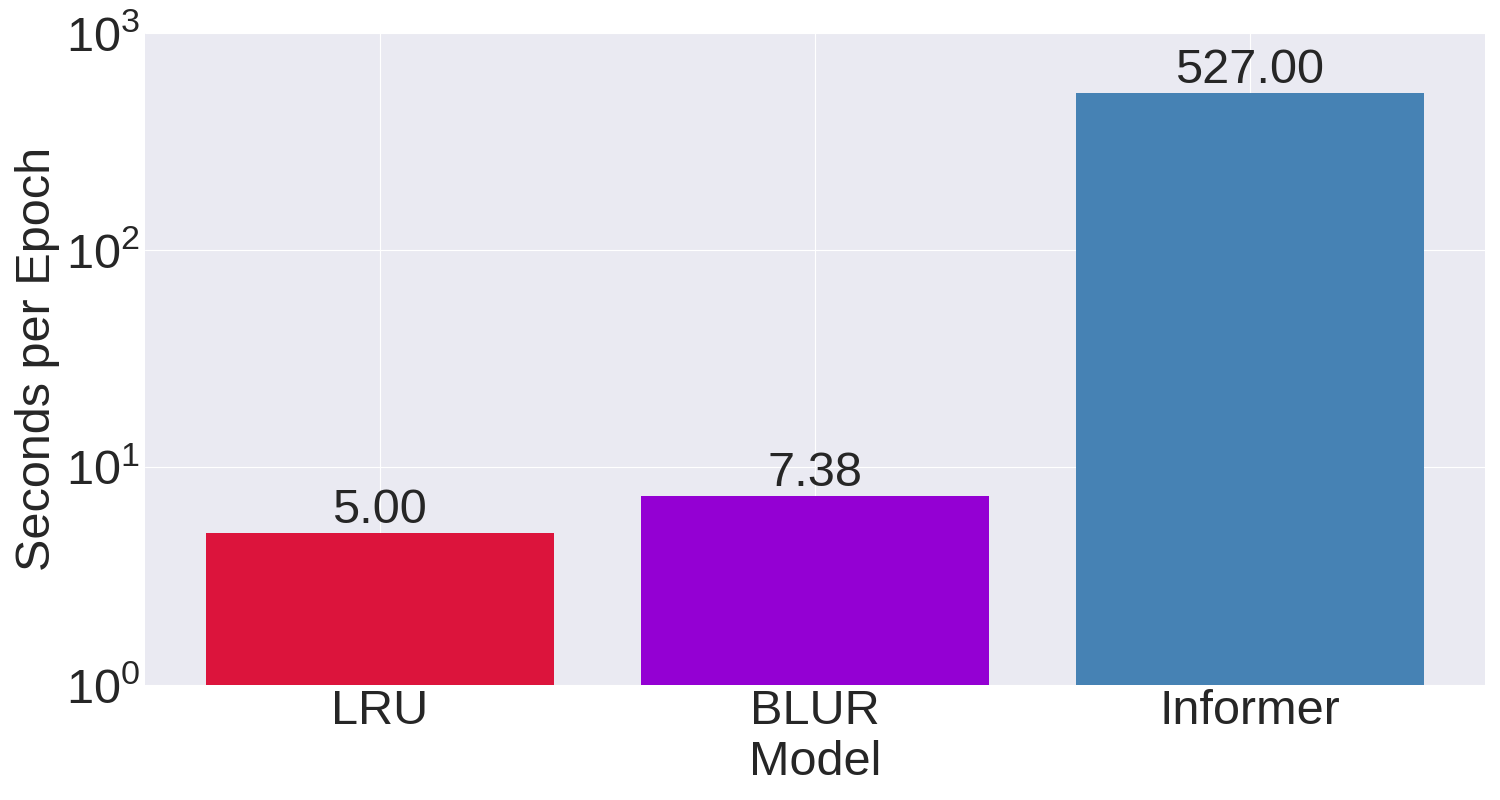}
    \caption{Training time costs per epoch for different methods with ETTh$_1$.}
    \label{fig:time_comparison_regression}
\end{figure}



\end{document}